\documentclass[acmsmall,screen,balance=false,nonacm]{acmart} %acmsmall   
\usepackage{graphicx}
\usepackage{amsfonts}
\usepackage{hyperref}
\usepackage{url}
\usepackage{booktabs}
\usepackage{nicefrac}
\usepackage{microtype}
\usepackage{comment}
\usepackage{lipsum}
\graphicspath{ {./} }
\usepackage[ruled,vlined,linesnumbered]{algorithm2e}
\usepackage{mathtools}
\usepackage{xcolor}
\usepackage{array}
\usepackage{cleveref}
\usepackage{enumitem}
\usepackage{wrapfig}
\usepackage{booktabs, tabularx}
\usepackage{mathtools,trimclip,stackengine,scalerel}
\usepackage{eucal}
\usepackage{caption}
\graphicspath{ {./images/} }
\usepackage{ifthen}
\usepackage{amsmath}
\DeclareMathOperator*{\argmax}{arg\,max}
\DeclareMathOperator*{\argmin}{arg\,min}
\usepackage{comment}
\usepackage{thmtools, thm-restate}
\usepackage{hyperref}
\usepackage{longtable}
\usepackage{pifont}% http://ctan.org/pkg/pifont
\usepackage{environ}
\usepackage{subcaption} % sub-figures

\SetKwInOut{Parameter}{Parameters}

\newcommand{\tool}{\texttt{BaVerLy}\xspace}

\newcommand{\TBD}[1]{\textcolor{red}{\bf TBD}}

\newcommand{\EXTENDEDVER}{-1} % set to -1 to enable extended , to to 1 to disable extended.

\SetKwProg{Fn}{Function}{:}{end}

\begin{document}

\title{Mini-Batch Robustness Verification of Deep Neural Networks}

\author{Saar Tzour-Shaday}
\orcid{0009-0005-9616-8944}
\affiliation{%
  \institution{Technion}
  \city{Haifa}
  \country{Israel}
}
\email{saartz@campus.technion.ac.il}

\author{Dana Drachsler-Cohen}
\orcid{0000-0001-6644-5377}
\affiliation{%
  \institution{Technion}
  \city{Haifa}
  \country{Israel}
}
\email{ddana@ee.technion.ac.il}

\begin{abstract}
Neural network image classifiers are ubiquitous in many safety-critical applications. 
However, they are susceptible to adversarial attacks. To understand their robustness to attacks, many local robustness verifiers have been proposed to analyze $\epsilon$-balls of inputs. Yet, existing verifiers introduce a long analysis time or lose too much precision, making them less effective for a large set of inputs. In this work, we propose a new approach to local robustness: \emph{group local robustness} verification. The key idea is to leverage the similarity of the network computations of certain $\epsilon$-balls to reduce the overall analysis time. We propose \tool, a sound and complete verifier that boosts the local robustness verification of a set of $\epsilon$-balls by dynamically constructing and verifying mini-batches.
\tool adaptively identifies successful mini-batch sizes, accordingly constructs mini-batches of $\epsilon$-balls that have similar network computations, and verifies them jointly. 
If a mini-batch is verified, all its $\epsilon$-balls are proven robust. Otherwise, one $\epsilon$-ball is suspected as not being robust, guiding the refinement. 
\tool leverages the analysis results to expedite the analysis of that $\epsilon$-ball as well as the analysis of the mini-batch with the other $\epsilon$-balls. We evaluate \tool on fully connected and convolutional networks for MNIST and CIFAR-10. Results show that \tool scales the common one by one verification by 2.3x on average and up to 4.1x, in which case it reduces the total analysis time from 24 hours to 6 hours. 
\end{abstract} 
\maketitle
\section{Introduction}
Neural networks are successful in many applications, including object detection, speech recognition, text generation and machine translation~\cite{YOLO,SEQ_TO_SEQ,SPEECH_RECOGNITION,MACHINE_TRANSLATION}. 
In particular, they are widely ubiquitous as image classifiers~\cite{ALEXNET}, playing a crucial role in safety-critical applications, such as autonomous cars~\cite{AUTONOMOUS_DRIVING1,AUTONOMOUS_DRIVING2,AUTONOMOUS_DRIVING3}, medical diagnosis~\cite{SKIN_CANCER,LUNG_PATTERN}, and surveillance systems~\cite{SURVEILLANCE1,SURVEILLANCE2}.
Guaranteeing the safety of these networks is imperative in these settings, especially in light of the recent European Regulations on Artificial Intelligence~\cite{eucaiwhitepaper}.

However, neural networks are known to be vulnerable to different kinds of attacks. One of the attacks that has drawn a lot of attention in recent years is adversarial example attacks~\cite{INTRIGUING_PROP,ADVERSARIAL_EXAMPLES,EXPLAINING_ADVERSARIAL_EXAMPLES,ADVERSARIAL_ATTACKS1,ADVERSARIAL_ATTACKS2,ADVERSARIAL_ATTACKS3,ADVERSARIAL_ATTACKS4}. An adversarial attack that targets an image classifier computes a small noise -- typically imperceptible to the human eye -- that leads the network to misclassify. To show the robustness of a neural network to these attacks, many robustness verifiers have been introduced~\cite{RELUPLEX,MIPVERIFY,AI2,DEEPPOLY,DEEPZ,SEMANTIFY_NN,MARABOU,MN_BAB,BETA_CROWN,NNV}. Most of them focus on proving the \emph{local robustness} of a given $L_\infty$ $\epsilon$-ball~\cite{MIPVERIFY,DEEPPOLY,DEEPZ,BETA_CROWN}, though some of them focus on other kinds of perturbations, such as other $L_p$ $\epsilon$-balls~\cite{CALZONE,COVERD,L2_PERTURBATION}, geometric perturbations~\cite{GEOMETRIC_PERTURBATIONS_SHARED_CERTIFICATES,GEOMETRIC_PERTURBATIONS2}, or global robustness~\cite{GLOBAL_ROBUST,VHAGAR}.

Despite the immense research on verifiers for determining the local robustness in a single $\epsilon$-ball, they still face challenges in providing formal guarantees to deep networks:
complete verifiers struggle to scale because of their exponential time complexity, while incomplete verifiers struggle to successfully verify robustness of deep networks because of their precision loss.
Additionally, typically network designers are not interested in the local robustness of a single $\epsilon$-ball. Ideally, they aim at understanding the local robustness in all ``relevant'' $\epsilon$-balls. Since the set of relevant $\epsilon$-balls does not have a formal characterization, it is often estimated as the set of $\epsilon$-balls around inputs in a given test set. Although these sets often contain similar inputs, most local robustness verifiers do not leverage this setting and verify $\epsilon$-balls one by one. An exception is works on shared certificates~\cite{SHARED_CERTIF,FANC}, which learn {verification templates} with the goal of expediting the analysis of unseen $\epsilon$-balls. 
However, they are not designed to directly leverage the given test set to reduce the overall analysis time. Additionally,
the template generation takes several hours and existing shared certification techniques focus on incomplete verification, and consequently they may not expose the true robustness level of a network. 

In this work, we consider the problem of \emph{group local robustness verification}. Given a network, a set of inputs, and a real number $\epsilon$, the goal is to determine for every input's $\epsilon$-ball whether it is robust or not while minimizing the overall analysis time. 
We focus on complete verification, because it enables to understand the robustness level of the network. In particular, it provides a faithful approach to compare the robustness levels of two networks. This problem is challenging since it requires to identify which $\epsilon$-balls can be analyzed together without leading to spurious adversarial examples and without increasing the verification's complexity. 
The latter may happen since the verification's complexity is exponential in the number of non-stable neurons (i.e., neurons for which the activation function exhibits nonlinearity). Generally, unifying $\epsilon$-balls may lead to increasing the number of non-stable neurons. In particular, unwise unification may lead to a significantly higher number of non-stable neurons, thus increasing the verification's complexity and making the overall analysis time longer than analyzing the $\epsilon$-balls one by one. 

To balance between verifying multiple $\epsilon$-balls and avoiding increased complexity as well as precision loss, we propose to verify \emph{mini-batches}. A mini-batch is a small subset of inputs for which the network performs similar computations. 
This concept is inspired by common machine learning training algorithms, which process data in mini-batches to significantly enhance computational efficiency (though their mini-batches need not consist of inputs with similar network computations).
Verifying a mini-batch can be encoded by a mixed-integer linear program (MILP), extending the encoding of a previous local robustness verifier for a single $\epsilon$-ball~\cite{MIPVERIFY}.
However, the naive extension suffers from higher verification's complexity as well as precision loss. To cope, we propose several ideas.
First, we begin the joint verification of a mini-batch in an intermediate layer of the network (like the generated templates of~\citet{SHARED_CERTIF,FANC}). Unifying in an intermediate layer enables to focus on the computations where the $\epsilon$-balls are \emph{perceived similar}, thereby the verification's complexity does not grow significantly and the overapproximation error is low.  
Second, we encode the mini-batch verification such that the MILP solver either determines that the mini-batch is fully verified, or detects an $\epsilon$-ball which may be not robust. This encoding enables a simple refinement: this $\epsilon$-ball is analyzed separately and the other $\epsilon$-balls continue their joint analysis.
Thus, the time spent on the analysis of a mini-batch is not wasted. 
Further, after separating the possibly non-robust $\epsilon$-ball from the mini-batch, its analysis and the analysis of the remaining batch leverage the analysis results of the previous mini-batch to terminate faster.
Third, we estimate the similarity of $\epsilon$-balls by the activation patterns of their center input. This approach is both  fast and, in practice, estimates well closeness of $\epsilon$-balls. Fourth, we learn the optimal mini-batch size throughout the analysis. In particular, it may start with larger mini-batches, consisting of the most similar $\epsilon$-balls and reduce the mini-batch sizes, when the remaining $\epsilon$-balls are further apart. This step relies on multi-armed bandit with the Thompson Sampling.

We implemented our approach in a system called \tool (a \textbf{ba}tch \textbf{ver}ifer for \textbf{l}ocal robustness). We evaluate \tool on fully connected networks and convolutional networks for MNIST and CIFAR-10. \tool boosts the verification time by 2.3x on average and up to 4.1x compared to one by one verification. In particular, it reduces the analysis time from 13 hours to 5 hours, on average.
We further show that learning the optimal mini-batch size boosts \tool by 2.5x. 
\section{Problem Definition}
\label{sec:preliminary}
In this section, we define our problem: group robustness verification. We begin with background on image classifiers and local robustness. We then define our problem and discuss existing approaches.

\paragraph{Image classifiers}
Image classifiers take an input image $x$ and determine which class from a set of classes $\mathcal{C}$ describes the object shown in the image.
For example, a CIFAR-10 classifier maps images to one of ten classes, e.g., a ship or a deer.
An image classifier implemented by a deep neural network (DNN) is a function $N : [0,1]^{d_{in}} \to \mathbb{R}^{d_{out}}$ composed from $L$ hidden layers $N = N_{L} \circ \ldots \circ N_1$.
The input to the first hidden layer, referred to as the input layer, is denoted by $z_0 = x \in [0,1]^{d_{in}}$, while the output of the last hidden layer, known as the output layer, is denoted by $z_L \in \mathbb{R}^{d_{out}}$.
Each hidden layer $N_i$ takes as input the output vector of the previous layer $z_{i-1}$ and returns a vector $z_i$.
To compute the output vector, it first executes an affine transformation $z'_{i} = W_i z_{i-1} + b_i$, where $W_i$ and $b_i$ are the layer's weight matrix and bias vector, respectively. 
This transformation is then followed by a nonlinear activation function.
We focus on piecewise-linear networks, whose predominant activation function is the Rectified Linear Unit (ReLU). The ReLU function, computing $z_i=ReLU(z'_i)$, is invoked component-wise and returns the maximum of each component and zero: $\forall j. \ ({z_{i}})_j = ReLU(({z'_{i}})_j)= \max (({z'_{i}})_j, 0)$.
The output of the last layer $N_L$ contains $d_{out}$ neurons, each returns the score of a unique class $c \in \mathcal{C}$ (where $|\mathcal{C}|=d_{out}$).
The process of passing an input $x \in [0,1]^{d_{in}}$ through the DNN to receiving the output $N(x) \in \mathbb{R}^{d_{out}}$ is called a feed-forward pass.
At the end of this process, the classification for $x$ is the class with the highest score: $c' = \argmax (N(x))$.

\paragraph{Local robustness}
To prove safety to adversarial attacks, many works focus on analyzing the \emph{local robustness} of a network classifier~\cite{AI2, DEEPPOLY, RELUPLEX, GPUPOLY, MIPVERIFY, IMAGESTARS, MARABOU2, BICCOS, NNV}.
The vast majority of works focuses on proving robustness in the $L_\infty$ $\epsilon$-ball of a given input.
Formally, given an input image $x \in \mathbb{R}^{d_{in}}$ and an $\epsilon \in \mathbb{R}^+$, the $L_\infty$ $\epsilon$-ball of $x$ is the set of all inputs that differ from $x$ by at most $\epsilon$, that is
$B_\epsilon^\infty(x) = \{ x' \: | \: \| x - x' \|_{\infty} = \max (| x_1 - x'_1 | , \ldots , | x_{d_{in}} - x'_{d_{in}} |) \leq \epsilon\}$.
A network classifier $N$ is locally robust in $B_\epsilon^\infty(x)$ if it classifies all its inputs the same:
$\forall x' \in B_\epsilon^\infty(x), \: \argmax (N(x)) = \argmax (N(x'))$.
In the following, we say that $N$ is robust in the $\epsilon$-ball (or the neighborhood) of $x$ if $N$ is locally robust in $B_\epsilon^\infty(x)$.
Local robustness has been shown to be NP-hard~\cite{RELUPLEX}, which stems from the nonlinear activation function's computations (e.g., the ReLUs).
Thus, existing local robustness verifiers balance between their precision and scalability.
This gave rise to two approaches: complete and incomplete verifiers.
Complete verifiers guarantee to determine whether an $\epsilon$-ball is robust but suffer from a long runtime, which increases as the network is deeper (i.e., has more layers).
In contrast, incomplete verifiers favor scalability and overapproximate the activation computations to expedite the analysis at the expense of precision loss, i.e., the verifier may fail to prove robustness for some robust $\epsilon$-balls.
Commonly, the deeper the network or the larger the $\epsilon$, the higher the precision loss, and thus the higher failure rate of incomplete verifiers.

%\section{Problem Definition}
\label{sec:problem_def}

\paragraph{Group local robustness}
While many complete verifiers propose ways to scale their analysis, they still struggle to scale.
In this work, we aim to leverage the practical scenario of local robustness: verifying local robustness of a \emph{set} of $\epsilon$-balls.
While, ideally, a network designer wishes to understand the local robustness in every input's $\epsilon$-ball (called global 
robustness), this is much more challenging and existing global robustness verifiers do not scale to the size of networks that local robustness verifiers scale.
Instead, it is common to ``estimate'' the global robustness of the network by evaluating its local robustness in the $\epsilon$-balls of 
 a set of inputs.
While there is no guarantee that the network is locally robust in unseen $\epsilon$-balls, this approach helps designers compare the robustness
of networks to adversarial attacks. 
We next formally define this problem and discuss existing approaches.

\begin{definition}[Group Local Robustness Verification]~\label{def:dataset_verification}
    Given a set of inputs $S\subseteq [0,1]^{d_{in}}$, a classifier $N: [0,1]^{d_{in}}\rightarrow  \mathbb{R}^{d_{out}}$, and
     $\epsilon \in \mathbb{R}^+$, 
    \emph{group local robustness verification} determines for every input $x\in S$ whether $N$ is locally robust 
    in its $\epsilon$-ball $B^\infty_\epsilon(x)$ while minimizing the overall analysis time.
\end{definition}

\paragraph{Existing approaches}
The most common approach to addressing our problem involves designing a verifier that analyzes the local robustness of an
$\epsilon$-ball around an input and invoke it on every input in $S$ one by one.
However, this approach does not leverage the similarity of the network computations, which leads to a long analysis time.
To mitigate this, several studies have proposed reusing analysis computations.
For instance,~\citet{FANC} generate and transform templates that capture symbolic shapes at intermediate network layers, allowing proof computations to be reused across multiple approximate versions of a network.
However, this technique is tailored for proof transfer across similar networks rather than across different inputs.
\citet{SHARED_CERTIF} propose the concept of shared certificates, which
leverages the proofs of certain $\epsilon$-balls to speed up the verification of other $\epsilon$-balls, through a two-step process: offline template generation and inference.
During the template generation, a large set of $\epsilon$-balls of training inputs (e.g., several thousands~\cite{SHARED_CERTIF_ARXIV}) are verified one by one. The intermediate analysis results (e.g., zonotopes or polyhedrons) are attempted to be generalized to \emph{templates} through clustering, convex-hull extension, and other expansion techniques. These templates, encoded in the box or the star domain~\cite{STAR_DOMAIN1,STAR_DOMAIN2}, are subsequently verified using an exact verifier.
At inference, an $\epsilon$-ball begins the analysis and after every layer in which templates were generated, it is checked whether its analysis result is contained in one of the templates.
If yes, the analysis terminates; otherwise, the analysis continues as usual.
While shared certificates have been shown successful, they are coupled to the chosen abstract domain, which limits them to incomplete verification.
Consequently, the network designer may not understand the actual local robustness in the given set of $\epsilon$-balls, which can lead to incorrect conclusion when comparing the robustness of networks to one another.
Additionally, the training time has high overhead (multiple hours). 
Further, the training procedure is invoked once and does not consider the $\epsilon$-balls that are later analyzed.
In all existing approaches, the local robustness analysis is performed $\epsilon$-ball by $\epsilon$-ball, as illustrated in~\Cref{fig::existing_approach}.
While shared certification analysis aims at reducing the joint analysis time, it assumes that future 
unseen $\epsilon$-balls have similar intermediate analysis results as the $\epsilon$-balls of the inputs in the training set.

\begin{figure*}[t]
    \centering
    \includegraphics[width=1\linewidth, trim=0 270 10 0, clip, page=8]{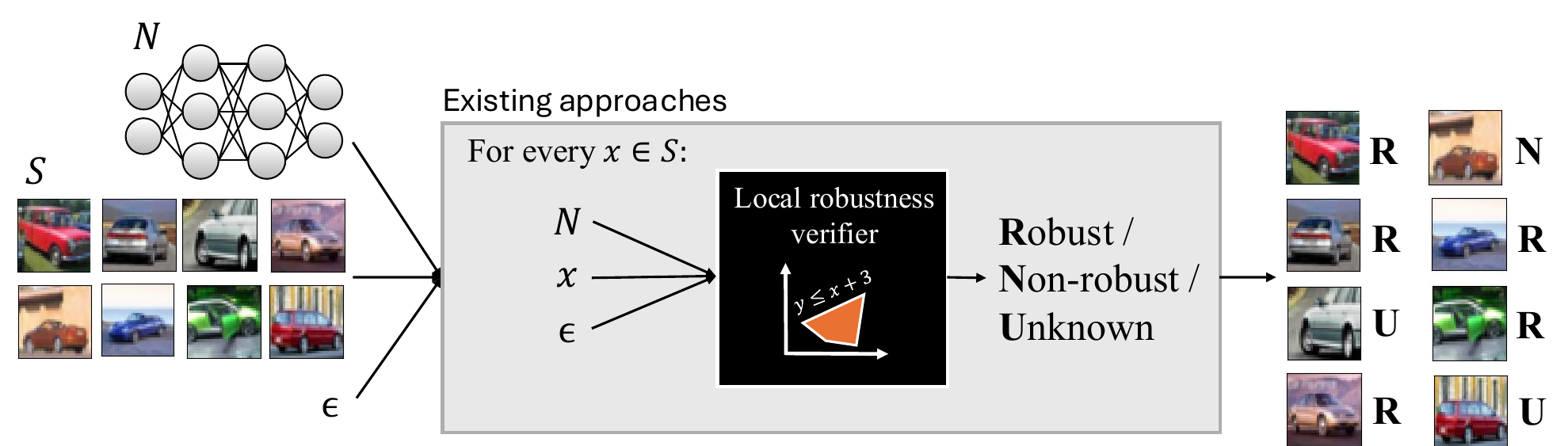} % left bottom right top
    \caption{Existing approaches analyze local robustness for each input's $\epsilon$-ball separately. Complete verifiers determine for each $\epsilon$-ball whether it is robust or not robust.
    Incomplete verifiers may also return unknown.}
    \label{fig::existing_approach}
\end{figure*}

\section{Overview: Mini-Batch Complete Verification}
\label{sec:insight}
In this section, we describe our approach to group local robustness verification: dynamically identifying small subsets of inputs -- called mini-batches -- whose $\epsilon$-balls are likely to be successfully verified together. 
At high-level, our verification relies on an abstraction-refinement procedure, similarly to~\citet{SYMBOLIC_INTERVAL,SYMBOLIC_INTERVAL_LINEAR_RELAXATION,CNN_ABSTRACT_REFINE,REFINE_ZONO}.
However, to minimize the analysis time, there are several inherent questions to address: (1)~what computations are abstracted given a mini-batch,
(2)~how to identify how many and which inputs to include in mini-batches and (3)~how to perform refinement. 
We next discuss our ideas to address these questions.

\begin{figure}[t]
    \centering
    \begin{subfigure}{0.48\linewidth}
        \centering
        \includegraphics[width=\linewidth, trim=0 340 340 0, clip, page=14]{images/figures.pdf} % left bottom right top
        \caption{Histogram of $L_\infty$ distance of MNIST closest pairs.}
        \label{fig::linf}
    \end{subfigure}
        \hspace{1em}
    \begin{subfigure}{0.48\linewidth}
        \centering
        \includegraphics[width=\linewidth, trim=120 370 120 10, clip, page=1]{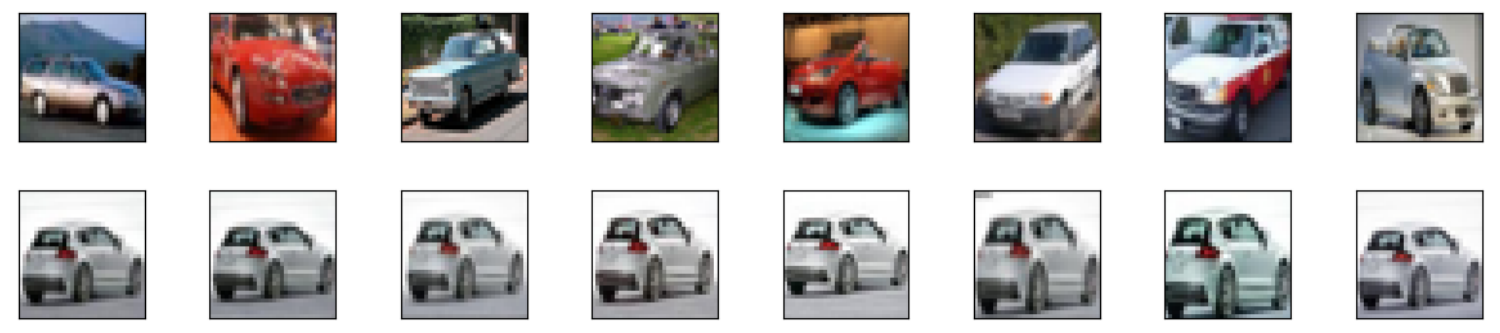}
        \caption{CIFAR-10's first 8 car images (top) and the first 8 images after sorting by SSIM (bottom).}
        \label{fig::ssim_sort}
    \end{subfigure}
    \caption{Input similarity by common metrics.}
\end{figure}

\subsection{Abstraction}
Given a set of inputs $S$, the most naive approach abstracts all their $\epsilon$-balls. 
However, it is very likely to fail proving robustness, especially if the inputs are classified differently or if the inputs are not very close to one another.
This is caused by two reasons. First, existing local robustness verifiers determine robustness by checking that all inputs in the given neighborhood are classified the same, thus abstracting inputs of different classes will lead the verifier to determine that the abstracted neighborhood is not robust. Second, the more different the inputs, the higher the overapproximation error and the more likely it contains spurious counterexamples, which will lead to failure.

A more natural approach abstracts the $\epsilon$-balls of inputs that are classified the same and are close, based on some similarity metric.
A natural candidate for similarity metric is the $L_\infty$ distance. However, even for the MNIST dataset, whose images are relatively similar, the images are not close enough to keep the overapproximation low. \Cref{fig::linf} shows a histogram of the $L_\infty$ distance of every MNIST test image and its closest image.  
It shows that the closest pair has distance of 0.33 (typically, the radius of the $\epsilon$-ball is much smaller), and that most pairs have significantly larger distance. 
For images, a better similarity metric is the structural similarity (SSIM) index~\cite{SSIM} (illustrated in \Cref{fig::ssim_sort}) or
LPIPS~\cite{LPIPS}. In~\Cref{sec:evaluation}, we show that abstracting in the input layer based on these metrics leads to a large analysis time and fails proving robustness for most $\epsilon$-balls. 
The reason is that despite the similarity, the abstraction still adds too many spurious inputs, which increases the analysis time at best and leads to spurious counterexamples at worst.

Instead, we rely on the following observation: given inputs classified to the same class, abstracting in a deeper layer loses less precision. 
Intuitively, the reason is that the output vectors of such inputs tend to become closer for deeper layers, where ultimately the last output vectors are equivalent in terms of the chosen classification. 
This observation is supported theoretically: the \emph{information bottleneck principle in deep learning}~\cite{BOTTLENECK} states that neural networks compress the input to enhance generalization.
That is, the input layer has a raw representation of the input $x$, where not all pixels contribute to its classification.
As the input is propagated through the network, each layer processes the representation of the previous layer, extracting the meaningful information while discarding irrelevant details. Thus, by abstracting $\epsilon$-balls in an intermediate layer, we can focus on the network computations where they are \emph{perceived similar}, which is more effective to expedite the  analysis. 
Inspired by this observation, 
we join the analysis of the batch's $\epsilon$-balls in an intermediate layer $\ell$.
This observation has also been leveraged in shared certificates~\cite{SHARED_CERTIF}, which are formed by templates in an intermediate layer.
Unlike shared certificates, we do not compute templates with the goal of expediting the analysis of future unseen $\epsilon$-balls, but batch the analysis of subsets of $\epsilon$-balls. \Cref{fig::our_approach} illustrates our approach, called \tool. Given a network, a set of inputs and an $\epsilon$, it iteratively forms batches (we explain how shortly).
For each batch, it verifies each $\epsilon$-ball separately up to layer $\ell$ (the choice of $\ell$ is described in~\Cref{sec:split}).
It then continues their analysis together.  
If the analysis succeeds, all $\epsilon$-balls of the batch are proven robust. If not, \tool identifies an $\epsilon$-ball that may be not robust (we explain how later). It then analyzes this $\epsilon$-ball separately and continues the joint analysis for the remaining batch. We explain later in this section why our refinement steps lead to very low overhead. 

\begin{figure*}[t]
    \centering
    \includegraphics[width=1\linewidth, trim=0 20 0 0, clip, page=9]{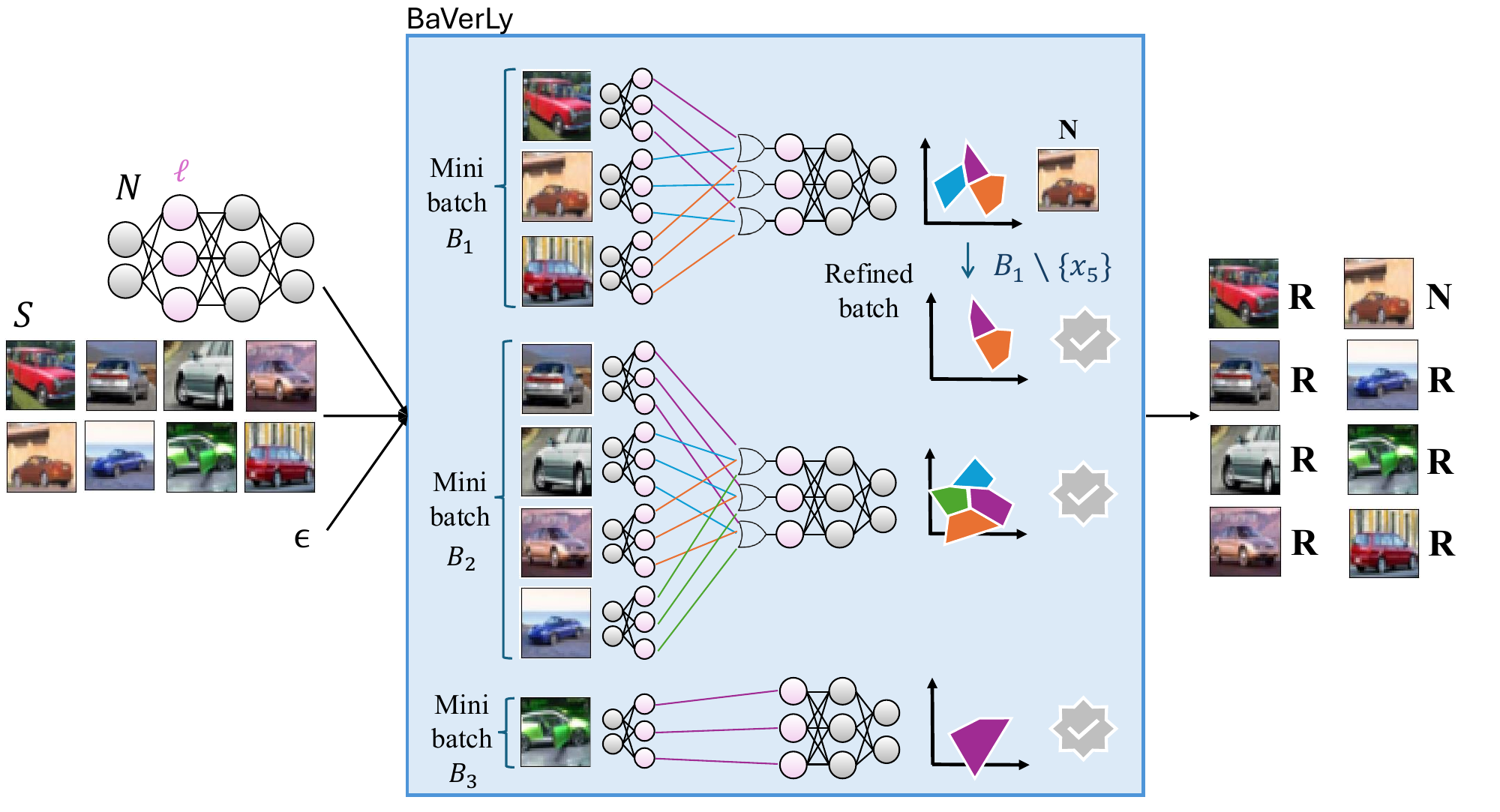} % left bottom right top
    \caption{Our approach for complete group local robustness verification forms mini-batches and analyzes them jointly starting from an intermediate layer. If a mini-batch fails, our verifier identifies a potentially non-robust $\epsilon$-ball and analyzes it separately. It then continues the analysis for the remaining $\epsilon$-balls in the mini-batch.}
    \label{fig::our_approach}
\end{figure*}

\subsection{Mini-Batches} Our second observation is that to balance well the precision-scalability trade-off, it is best to 
abstract to \emph{mini-batches}. That is, unify the verification of relatively small subsets of $\epsilon$-balls. While our algorithm works for any batch size, large batches increase the time overhead substantially because of their overapproximation error. Beyond balancing well the precision-scalability trade-off, there is an additional advantage in verifying mini-batches when the set of inputs $S$ is large: it enables \tool to learn the best mini-batch size based on previous mini-batches. This is possible because the group verification problem is invariant to the order in which $\epsilon$-balls are verified. Naturally, the best batch size depends on the inputs in the batch, thus our batches consist of inputs whose $\epsilon$-balls have similar network computations.
We next describe how \tool learns the best batch size from previous mini-batches and how it clusters inputs into a mini-batch of a selected size.

\paragraph{Learning the mini-batch size}
Choosing a good batch size is challenging. 
On the one hand, too large mini-batches can substantially increase the analysis time due to the overapproximation error. 
On the other hand, too small batches may also increase the analysis time, similarly to verifying the $\epsilon$-balls one by one.
The best batch size also depends on the inputs in $S$: the more similar inputs in $S$, the more effective larger mini-batches are. 
Even if \tool begins by grouping the most similar $\epsilon$-balls to relatively large mini-batches, as the analysis progresses, the remaining $\epsilon$-balls are likely to have more distant network computations, for which smaller mini-batches are more effective. 
We rely on an adaptive approach to learn the optimal batch size on the fly.
Our approach adopts a strategy from reinforcement learning (RL), where an agent learns a policy that maximizes the received reward.
In our setting, the optimal policy's goal is to predict batch sizes that enable \tool to minimize the analysis time per input in $S$.
We formalize this goal as increasing the \emph{batch velocity}, 
that is the number of $\epsilon$-balls which were proven robust within the batch verification, divided by the analysis time of the batch (excluding its refinements).
We note that the concept of partitioning a local robustness task by predicting the subparts that maximize the proof velocity has been proposed by others~\cite{VEEP}, however they focus on verifying the local robustness of a single semantic feature neighborhood.
An inherent dilemma of an RL agent is the \emph{exploration-exploitation trade-off}. In our context, this means that whenever our RL agent chooses a mini-batch size it can choose between exploring new, potentially effective mini-batch sizes (which may be discovered as less effective) or exploiting batch sizes that have been shown to be reasonably effective (which may lead to not discovering more effective batch sizes).
In our setting, this problem becomes even more challenging since the $\epsilon$-balls can vary in the location of their perturbations and in the similarity level of their network computations.
To cope, we frame the problem of predicting the best mini-batch size as \emph{a multi-armed bandit} (MAB) scenario
and rely on the Thompson Sampling~\cite{THOMPSON_SAMPLING} that seamlessly balances reward maximization (exploitation) and variance minimization (exploration).
Technically, we introduce a different arm for each batch size. For each batch size, we learn a distribution that converges to the velocity of batches with this size. The distributions are updated throughout the execution of \tool.
\Cref{fig::batch_size} illustrates our approach for learning the best batch size.

\begin{figure*}[t]
    \centering
    \includegraphics[width=1\linewidth, trim=0 10 0 0, clip, page=10]{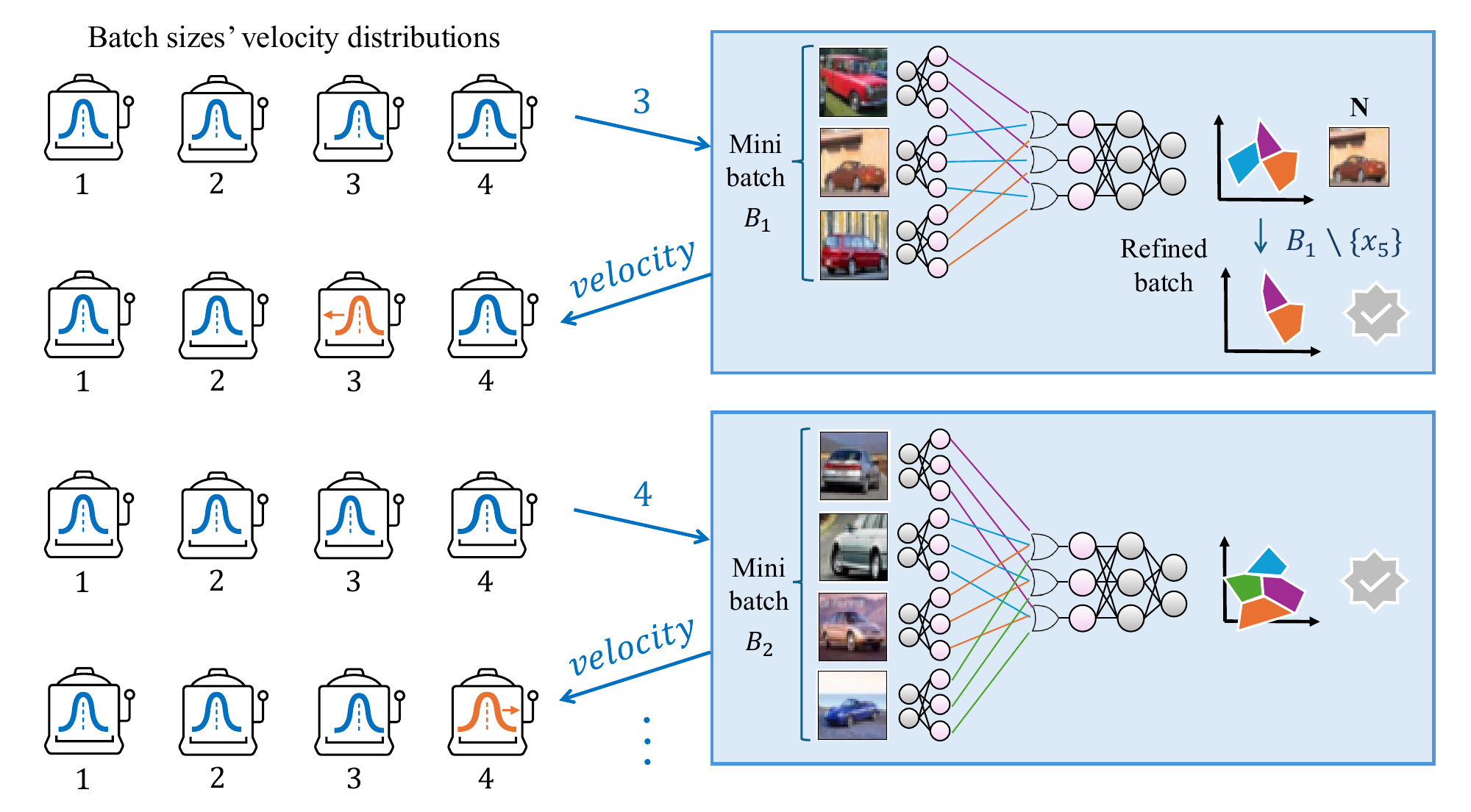} % left bottom right top
    \caption{Our approach for learning the best mini-batch size. We phrase the problem as a multi-armed bandit and learn the velocity distribution of every batch size. At every iteration, \tool samples the mini-batch size by the Thompson Sampling. It then constructs a mini-batch and verifies it. Afterward, the batch velocity is computed and the respective batch size's distribution is updated.}
    \label{fig::batch_size}
\end{figure*}

\begin{figure*}[t]
    \centering
    \includegraphics[width=0.7\linewidth, trim=30 355 350 10, clip, page=11]{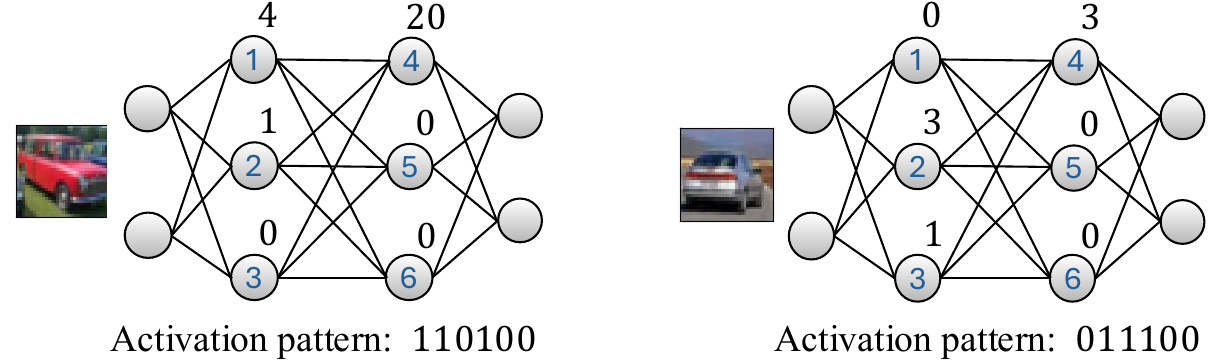} % left bottom right top
    \caption{Given an input, an activation pattern abstracts the computation of the intermediate neurons by a boolean vector whose $i^\text{th}$ entry is one, if the $i^\text{th}$ neuron outputs a positive value, and zero otherwise.}
    \label{fig::act_pat}
\end{figure*}

\paragraph{Constructing effective mini-batches}
The next question is how to construct a mini-batch, given the selected size $k$.
As mentioned, \tool aims at grouping the $k$ inputs whose $\epsilon$-balls exhibit the most similar network computations. This raises two questions: (1)~how to identify the $\epsilon$-balls with the closest network computations and (2)~how to identify $k$ such $\epsilon$-balls? 
The first question is particularly challenging since \tool does not know what the network computations of a given $\epsilon$-ball are without performing its analysis. Obviously, if \tool performed the analysis of every $\epsilon$-ball separately, there would be no point in the mini-batch analysis. Instead, we estimate the closeness of two $\epsilon$-balls by the similarity of the network computations for their center inputs. This is obtained by first running every input in $S$ through the network, which introduces negligible overhead. We then abstract the network computation of every input by its \emph{activation pattern}. An input $x$'s activation pattern is a boolean vector consisting of a bit for every ReLU neuron in the network. A bit is one if the respective neuron is \emph{active} (i.e., positive) when propagating $x$ through the network, and zero otherwise. \Cref{fig::act_pat} exemplifies the activation patterns of two inputs. 
The advantage of relying on the similarity of activation patterns rather than metrics at the input layer is that it estimates the increase in the verification's complexity caused by the unification of $\epsilon$-balls.
At high-level, the verification's complexity is exponential in the number of unstable ReLUs. Given a neighborhood of inputs, the unstable ReLU neurons are the neurons whose weighted sum inputs can be both positive and negative. Namely, these ReLUs can be both active and inactive, making the ReLU computation nonlinear in this neighborhood. The fewer the unstable ReLU neurons, the lower the verification's complexity. 
While the activation pattern of an input does not indicate which neurons are unstable in its $\epsilon$-ball (since it does not consider every possible input in the $\epsilon$-ball), inputs which differ in the activation state of a certain neuron, imply that this neuron must be unstable if we unify their $\epsilon$-balls. For example, consider neuron $1$ in~\Cref{fig::act_pat}. It is in active state for the first image and in inactive state for the second image. If we unify these images' $\epsilon$-balls into one neighborhood, neuron 1 must be unstable. On the other hand, for neurons~$2$ and $6$ in~\Cref{fig::act_pat}, both images have the same state (for both, neuron $2$ is active and neuron $6$ is inactive). Thus, although it may be that these neurons are unstable if we unify these images' $\epsilon$-balls, it may also be that these neurons are stable. 
We note that if a certain neuron has different active/inactive states for two images, it does not necessarily mean that if we unify the two images' $\epsilon$-balls the verification's complexity increases, since it could be that one of their $\epsilon$-balls makes this neuron unstable.
This is our motivation for preferring to unify $\epsilon$-balls whose center inputs have close activation patterns.
We measure the distance of two activation patterns by their Hamming distance (i.e., the number of different bits). For example, the Hamming distance of the two images in~\Cref{fig::act_pat} is $2$, since their first and third bits are different.
Relying on activations to identify network similarities has been proposed in prior work. For example, \citet{INTRIGUING_PROP} show that activation values of neurons in the hidden layers encode semantic information about the features seen in the image.
In particular, inputs that share many common features tend to have close activation patterns.
\citet{ALEXNET} show that the network perceives images as semantically similar when their deeper layer activations are proximal, even when the images' pixels differ substantially.

We now explain how \tool forms a mini-batch of up to $k$ inputs.
A natural idea is to rely on clustering algorithms, such as K-Means~\cite{KMEANS}. However, most clustering algorithms are effective in clustering a set of elements into a certain number of clusters, whereas we are interested in clusters of \emph{given sizes} and the given sizes \emph{change} during the execution of \tool.
We thus rely on \emph{Hierarchical Clustering (H-Cluster)}~\cite{HCLUSTER}. H-Cluster has been proposed in phylogenetics for revealing the evolutionary ancestry between a set of genes, species, or taxa. Given a set of vectors, the H-Cluster greedily constructs a diagrammatic representation of the clusters hierarchy, called a \emph{dendrogram}.
\Cref{fig::dendrogram} shows an example of a dendrogram over six activation patterns, where the number of an internal node is the maximum Hamming distance of the activation patterns in its subtree.  
For example, the Hamming distance between $x_3$ and $x_4$ is $420$ and the distance between $x_5$ and $x_6$ is $477$. The largest Hamming distance between every pair of inputs in $\{x_3,x_4,x_5,x_6\}$ is $595$.
\tool transforms the dendrogram into a binary tree.
Every leaf corresponds to an input $x\in S$ and inner nodes represent clusters consisting of all leaves in their subtrees.
Figure \ref{fig::binary_tree} shows an example of this binary tree, where the numbers in the inner nodes are the size of their cluster.
This binary tree enables \tool to efficiently track the remaining inputs to verify and construct batches of given sizes. 
To construct a batch of size up to $k$, it traverses the tree in pre-order, stopping at the first node whose number is less than or equal to $k$.
Then, it forms the batch by collecting the leaves and removes this subtree.
The search time complexity is $O(\log{|S|})$ on average and $O(|S|)$ in the worst case (since the binary tree can be unbalanced).

\begin{figure*}[t]
    \centering
    \begin{subfigure}[t]{0.45\textwidth}
        \centering
        \includegraphics[width=0.65\linewidth, trim=0 310 720 0, clip, page=3]{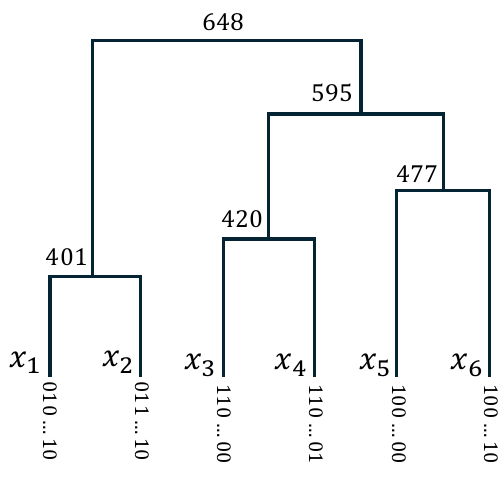} % left bottom right top
        \caption{Hierarchical clustering based on the Hamming distance of the activation patterns.}
        \label{fig::dendrogram}
    \end{subfigure}%
    \hspace{0.5cm}
    \begin{subfigure}[t]{0.45\textwidth}
        \centering
        \includegraphics[width=0.8\linewidth, trim=0 270 580 0, clip, page=2]{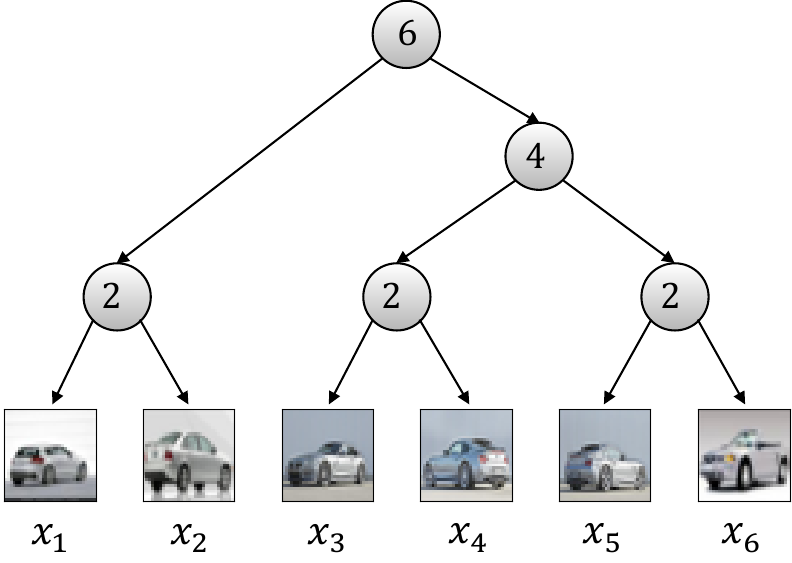} % left bottom right top
        \caption{The corresponding binary tree, for forming batches of up to a given size.}
        \label{fig::binary_tree}
    \end{subfigure}
    \caption{Illustration of the H-Cluster algorithm.}
    \label{fig::hcluster}
\end{figure*}

\paragraph{Refinement}
We next discuss what \tool does if the verification of a mini-batch finds a counterexample. 
We remind that our mini-batch verification analyzes every $\epsilon$-ball of the mini-batch separately until some layer
$\ell$ and then analyzes the $\epsilon$-balls jointly till the output layer.
A straightforward encoding of a mini-batch abstracts the $\epsilon$-balls' outputs at layer $\ell$ into the bounding box containing all these outputs.
However, this increases the input space of layer $\ell+1$ significantly, possibly including spurious adversarial examples, which will fail the verification. Even if the verifier finds a true adversarial example, it can require time to determine this is the case. Instead, we propose to encode the \emph{union} of the $\epsilon$-balls' outputs at layer~$\ell$. Technically, we define MILP constraints capturing a disjunction that restricts the inputs to layer $\ell+1$ to exactly the $\epsilon$-balls' outputs at layer~$\ell$. Our constraints associate a binary variable for each $\epsilon$-ball of the batch. If an adversarial example is found, one of these binary variables is one, indicating that the adversarial example belongs to the respective $\epsilon$-ball. Because the mini-batch analysis begins from an intermediate layer (and not the first layer), this adversarial example may be spurious.
Thus, \tool excludes the $\epsilon$-ball that may not be robust and analyzes it separately from the first layer to the last one. 
Accordingly, it determines whether it is robust or not. 
Then, \tool continues the analysis for the remaining mini-batch.
An advantage of our refinement step is that neither the analysis of the $\epsilon$-ball that may not be robust 
nor the analysis of the remaining mini-batch begin from scratch.
Both use the previous analysis computations to prune their search space. These two ideas enable our refinement step to introduce very low overhead.

\section{\tool: A Batch Verifier for Local Robustness} 
\label{sec:overview}

In this section, we present our group local robustness verifier.
We describe its algorithm and then its components.
\ifthenelse{\EXTENDEDVER<0}{\Cref{sec:runex}}{\citet[Appendix A]{Tzour25}} provides a running example.

\subsection{\tool's Algorithm}
\begin{algorithm}[t]
    \caption{\tool ($N$, $S$, $c$, $\epsilon$)}
    \label{algo::batch_verif}
    \DontPrintSemicolon
    \KwIn{A neural network $N$, a set of inputs $S$, a class $c$, and an epsilon $\epsilon\geq0$.}
    \KwOut{A dictionary $\texttt{is\_robust}$, reporting \emph{Robust} or \emph{Non-robust} for every $x \in S$.}
    $\text{is\_robust} = []$ \tcp*{Dictionary mapping inputs to \emph{Robust} or \emph{Non-Robust}}
    $\text{AP} = []$ \tcp*{Dictionary mapping inputs to activation patterns}
    \For{$x\in S$}{\label{line::filterb}
        \lIf{$\argmax N(x) \neq c$}{$\text{is\_robust}[x]$ = \emph{Non-Robust};  $S = S \setminus \{x\}$}

    }\label{line::filtere}
    $\ell$, $S$, is\_robust = learnSplitLayer($N$, $S$, $c$, $\epsilon$, is\_robust) \; \label{line::learn_split_call}
    
    \lFor{$x\in S$}{\label{line::apb}
        $\text{AP}[x]$ = activatation\_pattern($N$, $x$)  \label{line::ape}
    }
    $\mathcal{T}$ = getBinaryTree(H-Cluster($\text{AP}$)) \label{line::hclust} \;
    
    $\mathcal{MAB}$ = initialize(MAX\_BATCH\_SIZE, $\rho$, {BUCKET\_SIZE}) \label{line::init_mab} \;
    \While{$\mathcal{T}\neq \bot$} { \label{line::main_loop}
        $k$ = $\mathcal{MAB}$.getMiniBatchSize() \label{line::recommend} \;
        $B$ = constructBatch($\mathcal{T}$, $k$) \label{line::search_batch} \;
        $k$ = $|B|$ \label{line::actual_batch} \tcp*{Actual batch size}
        $\text{bounds} = []$ \tcp*{Dictionary mapping inputs to theirs bounds up to $N_\ell$}\label{line::prefix_boundsb}
      $\text{start\_time} = \text{current\_time}()$\label{line::start_time} \;
        \For (\tcp*[f]{Compute bounds up to $N_\ell$ for each input}){$x \in B$} {
            $\text{bounds}[x]$ = \text{MILPBounds}($N_\ell\circ\ldots \circ N_1$, $x$, $\epsilon$) \label{line::prefix_boundse} \;
        }
        
        $\text{MILP\_SUF}$ = \text{MILPBatch}($N_L\circ \ldots \circ N_{\ell+1}$, $\{\text{bounds}[x][\ell]\mid x\in B\}$, $c$)\; \label{line::verify_batch_begin}
        $\text{cex}$ = \text{MILPSolve}($\text{MILP\_SUF}$)   \tcp*{Verify the batch}    \label{line::verify_batch_mid} 
        $\text{total\_time} = \text{current\_time}() - \text{start\_time}$ \label{line::up_time} \;
        
        \While{$\text{cex} \neq \bot$}{     \label{line::verify_batch} 
            Let $x$ be the input whose variable $I_x$ is $1$ in $\text{cex}$\tcp*{$B_\epsilon^\infty(x)$ may be non-robust}\label{line:detect}
            $\text{cex}$ = \text{MIPVerify}($N$, $x$, $\epsilon$, $\text{bounds}[x]$) \label{line::refine} \tcp*{Verify $B_\epsilon^\infty(x)$}
            $\text{is\_robust}[x]$= $\text{cex} == \bot$? Robust : Non-Robust \;\label{line::non_robust}
            $B = B \setminus \{x\}$\tcp*{Update the batch}\label{line::upbatch}
            \lIf{$B==\emptyset$}{break}\label{line::break}
             $\text{start\_time} = \text{current\_time}()$\label{line::start_time2} \;
            $\text{MILP\_SUF}$ = $\text{MILP\_SUF}::\{I_x = 0\}$\tcp*{Ignore the $\epsilon$-ball of $x$}\label{line:addcons}
            $\text{cex}$ = \text{MILPSolve}($\text{MILP\_SUF}$)\tcp*{continue with the rest}\label{line:addconssolve}
            $\text{total\_time} = \text{total\_time} + (\text{current\_time}() - \text{start\_time})$ \label{line::up_time2} \;
        }
        \lFor(\tcp*[f]{All $\epsilon$-balls are robust}){$x \in B$}{ 
        $\text{is\_robust}[x]$= Robust \label{line::robust_b}
        }\label{line::robust_e}
        $\mathcal{MAB}[k]$.updateDistribution($\frac{|B|}{\text{total\_time}}$) \label{line::velocity_update} \tcp*{Update based on the velocity}
    }
    \Return{$\text{is\_robust}$}
\end{algorithm}

\tool takes a network classifier $N$, a set of inputs $S$, a class $c$, and an $\epsilon$.
It returns a dictionary \texttt{is\_robust} mapping every input in $x\in S$ to \emph{Robust}, if $N$ classifies all inputs in the $\epsilon$-ball $B_\epsilon^\infty(x)$ as~$c$, or to \emph{Non-Robust} otherwise.
It begins by passing each $x \in S$ through $N$ (Lines \ref{line::filterb}--\ref{line::filtere}). If $x$ is not classified as $c$ by $N$, \tool tags it as \emph{Non-Robust} and removes it from $S$. 
Then, \tool chooses the split layer $\ell$ by calling \texttt{learnSplitLayer} (\Cref{line::learn_split_call}), defined in~\Cref{sec:split}. At high-level, \texttt{learnSplitLayer} verifies $L-1$ $\epsilon$-balls of random inputs from $S$, where each verification splits at a different layer. Accordingly, it chooses for $\ell$ the layer with the minimal verification time.   
It removes the analyzed inputs from $S$ and records their  status in \texttt{is\_robust}.
Then, \tool computes the activation patterns (defined in~\Cref{sec:cluster}) for every input in $S$  and adds it to a dictionary \texttt{AP} (\Cref{line::apb}).
Then, it constructs the clusters' dendrogram and its binary tree $\mathcal{T}$ (\Cref{line::hclust}), described in~\Cref{sec:cluster}, storing the unhandled inputs.
It then initializes the multi-armed bandit agent $\mathcal{MAB}$ (\Cref{line::init_mab}), which learns the optimal mini-batch size (described in~\Cref{sec:mab}).
$\mathcal{MAB}$ is initialized with the maximal mini-batch size \texttt{MAX\_BATCH\_SIZE} (a hyper-parameter) and 
other arguments, described later. 

Then, the main loop runs while there are inputs in the tree $\mathcal{T}$ (\Cref{line::main_loop}). At each iteration, $\mathcal{MAB}$ recommends a batch size $k$ (\Cref{line::recommend}).
\tool then searches $\mathcal{T}$ for a mini-batch $B$ of up to size $k$, constructs it (\Cref{line::search_batch}) and updates $k$ to be the actual batch size (\Cref{line::actual_batch}). For every input $x$ in the mini-batch, \tool analyzes its $\epsilon$-ball separately up to layer $\ell$ (Lines~\ref{line::prefix_boundsb}--\ref{line::prefix_boundse}).
This analysis computes for every $B_\epsilon^\infty(x)$ and every layer real-valued bounds, using  MILPs (described in~\Cref{sec:milp}).
The bounds are stored in a dictionary \texttt{bounds} mapping input to its bounds, i.e., \texttt{bounds[$x$]} 
is a list of lists, where the $i^\text{th}$ list contains the bounds of the neurons in layer $N_i$.

Next, \tool verifies the mini-batch. It begins by encoding the mini-batch verification problem over all inputs in $B$ from layer $N_{\ell+1}$ to the output layer (\Cref{line::verify_batch_begin}), where 
the input space to layer $N_{\ell+1}$ is the union of the bounds of layer $N_\ell$ over all inputs. 
We describe the encoding in~\Cref{sec:milp}.
At high-level, it consists of constraints such that if they are satisfiable, there is a vector $v$, which is not classified as $c$, contained within the bounds of $N_\ell$ of some input $x\in B$. This vector is possibly an indication to an adversarial example within $B_\epsilon^\infty(x)$. This is the case if there is $x'\in B_\epsilon^\infty(x)$ such that $v$ is the output of layer $N_\ell$ for $x'$, i.e., $v=N_\ell\circ\ldots \circ N_1(x')$. Because the bounds provide an overapproximation, it can happen that there is no such $x'$ in $B_\epsilon^\infty(x)$.
Our encoding of the union relies on a binary variable $I_x$, for each $x\in B$, and on constraints that guarantee that if the MILP is satisfiable, exactly one $I_x$ is equal to $1$. 
If $I_x=1$ for $x\in B$, then the satisfying assignment includes a vector $v$ within the bounds of $N_\ell$ for $B_\epsilon^\infty(x)$ that is not classified as $c$. This union encoding enables \tool to not lose precision at the layer $N_\ell$ and identify \emph{which} input from $B$ \emph{may} be non-robust.
Our MILP encoding is submitted to a solver, which returns an assignment \texttt{cex} or $\bot$ (\Cref{line::verify_batch_mid}).

If the MILP solver finds an assignment \texttt{cex}, \tool begins a loop to refine the batch while there exists a counterexample (\Cref{line::verify_batch}). 
At each iteration, it first identifies the input $x$ whose bounds contain the counterexample $v$ (\Cref{line:detect}), i.e., the input $x$ whose binary $I_x$ is equal to one. 
Then, it verifies the local robustness of $N$ in $x$'s $\epsilon$-ball separately (\Cref{line::refine}). This analysis is identical to MIPVerify~\cite{MIPVERIFY}, on which our MILP encoding builds, except that \tool leverages the prior analysis and provides it with the bounds for all layers up to $\ell$ to expedite the analysis (explained in~\Cref{sec:milp}). This analysis is complete, and thus \tool concludes whether $N$ is robust in $x$'s $\epsilon$-ball, depending on whether the MILP solver finds a counterexample (\Cref{line::non_robust}). Then, \tool removes $x$ from the batch $B$ 
(\Cref{line::upbatch}). If $B$ is empty, it breaks from the inner loop (\Cref{line::break}).
Otherwise, \tool updates the MILP to ignore $x$'s bounds by forcing $I_x=0$ (\Cref{line:addcons}) and calls the solver to check if there is another counterexample (\Cref{line:addconssolve}). Note that this analysis continues from the point that the solver terminates and does not start from scratch the computation. 

The inner loop terminates when there is no counterexample, implying that all inputs in $B$ are robust. Thus, \tool updates their robustness status in \texttt{is\_robust}  (\Cref{line::robust_b}). It then updates the $\mathcal{MAB}$ agent with the velocity of this batch (\Cref{line::velocity_update}). 
The velocity is the number of $\epsilon$-balls proven robust as part of the batch (i.e., the size of $B$ at the end of the inner loop) divided by the overall analysis time of the batch. This analysis time is the total time of (1)~computing the bounds of all inputs in the initial $B$ up to layer $\ell$, (2)~computing the bounds of the batch starting from layer $\ell+1$, and (3)~looking for counterexamples in the batch.
This time excludes the time to prove robustness of $\epsilon$-balls suspected as not robust (\Cref{line::refine}), since they are not proven as part of the batch.

\paragraph{Beyond group verification}
While we focus on group local robustness verification, we believe our approach can expedite the verification of other properties. 
It is applicable to any safety property whose input space can be partitioned into subspaces.
For example, semantic feature neighborhoods often contain a large set of inputs that cannot be analyzed jointly and are split statically or dynamically (e.g.,~\cite{DEEPPOLY,VEEP}).
For such neighborhoods, \tool can be extended to get as input the subspaces (explicitly or symbolically). Then, it constructs mini-batches and verifies them as described. 
If a mini-batch is not robust, our union encoding (relying on the binary variables $I_x$) enables \tool to identify the subspace that may not be robust, analyze it separately, and continue verify the rest of the mini-batch. A key advantage of \tool is that continuing verifying a mini-batch, after removing a subspace (by setting its binary variable $I_x=0$), completes very fast, since it relies on the analysis of the original mini-batch.

\subsection{Batch Local Robustness Verification by Mixed Integer Linear Programming}\label{sec:milp}
In this section, we present how we rely on mixed-integer linear programming (MILP) for batch local robustness verification.
\tool relies on three MILPs: for computing bounds on the neurons in the early layers, for verifying a batch in the subsequent layers and for verifying the robustness of $\epsilon$-balls suspected as non-robust. Our MILPs rely on a prior encoding of local robustness verification of a single $\epsilon$-ball~\cite{MIPVERIFY}.
We next describe it and our encodings.
  
\paragraph{MIPVerify} MIPVerify~\cite{MIPVERIFY} is a verifier for determining the local robustness of a neural network using MILP.
It takes a network classifier $N$, an input $x\in [0,1]^{d_{in}}$ and its class~$c$, and an $\epsilon>0$.
It determines whether $N$ is robust in the $L_\infty$ $\epsilon$-ball of $x$ or not, in which case it returns an adversarial example. MIPVerify is sound and complete. Its complexity is exponential in the number of unstable ReLU neurons.
It can analyze classifiers with fully connected, convolutional, and max-pooling layers. We describe its constraints for fully connected layers, but our implementation supports the other layers.
MIPVerify begins by encoding the input layer's neurons with variables $z_{0,1},\ldots,z_{0,{d_{in}}}$ and enforcing the neighborhood with linear constraints:
$\forall m\in [d_{in}].\ z_{0,m} \geq \max(0,x_m-\epsilon)\land z_{0,m} \leq \min(1,x_m+\epsilon)$.
Then, it constructs the constraints capturing the network computations layer by layer.
For every layer $i$ with $m_i$ neurons, it adds $m_i$ linear constraints for capturing the affine computation:
$z'_{i} = W_i z_{i-1} + b_i$. Technically, the vector $z'_{i}$ is captured by $m_i$ variables $z'_{i,1}$,\ldots,$z'_{i,m_i}$.
Then, it computes real-valued lower and upper bounds $l_{i,m},u_{i,m}$  for $z'_{i,m}$ via optimization. 
This is computed by taking the constraints of all layers up to layer $i-1$ and for each neuron $m\in[m_i]$, solving two optimization problems (unless skipped by heuristics), one with objective $u_{i,m}=\max{z'_{i,m}}$ and the other one with objective $l_{i,m}=\min{z'_{i,m}}$.
Then, if $l_{i,m} \geq 0$, the neuron is \emph{active}, i.e., its function is the identity function: 
$z_{i,m}=z'_{i,m}$.
Similarly, if $u_{i,m} \leq 0$, the neuron is \emph{inactive}, i.e., its function is the constant 0: $z_{i,m}=0$.
Otherwise, $l_{i,m} < 0 \wedge u_{i,m} > 0$, the neuron is \emph{unstable}, i.e., its function is piecewise linear and thus it is not expressible as a single linear constraint.
To encode the ReLU computation precisely, MIPVerify introduces a binary variable $a_{i,m}$ that captures the two possible
 states and adds four constraints over $a_{i,m}$, $z'_{i,m}$ and the bounds $l_{i,m}, u_{i,m}$.
After generating the constraints of all layers (the output layer has no ReLUs but its bounds are computed), MIPVerify adds an objective function and a constraint whose goal is to find the minimum adversarial perturbation that is misclassified:
$\min_{x'} \Vert x - x'\Vert_\infty \quad \text{s.t.} \quad z_{L,c} \leq \max_{c' \neq c} z_{L,c'} $, where $x'=(z_{0,1},\ldots,z_{0,{d_{in}}})$ and $x=(x_1,\ldots,x_{d_{in}})$.
It then submits all constraints to a MILP solver.
If the MILP solver determines that the set of constraints is infeasible, then $N$ is robust in this $\epsilon$-ball.
If it finds a satisfying assignment, the values  $z_{0,1},\ldots,z_{0,d_{in}}$ form an adversarial example. 

\paragraph{\tool's MILPs} \tool relies on this MILP encoding for three tasks: 
(1)~for computing the bounds up to layer $\ell$ (Lines~\ref{line::prefix_boundsb}--\ref{line::prefix_boundse}),    
(2)~for batch verification (\Cref{line::verify_batch_begin}, \Cref{line:addcons}), and (3)~for verifying the local robustness of an $\epsilon$-ball suspected as not robust (\Cref{line::refine}). We next describe these MILPs.

\paragraph{Bound computation}
\Cref{algo::milpbounds} shows the bound computation for every neuron in the first $\ell$ layers of $N$, given the $\epsilon$-ball of an input $x$.
It first adds constraints bounding each input neuron within its interval, based on $x$ and $\epsilon$ (\Cref{line:init}).  
Then, for every layer $i$, it iterates the neurons and, for each, computes a lower and an upper bound on the affine function (Lines~\ref{line:lower}--\ref{line:upper}) by calling a MILP solver with all current constraints. Afterwards, it adds the layer's constraints  (similar to~\Cref{batch_encoding_4,batch_encoding_6}, but with respect to $l_{i}$, $u_{i}$ instead of $L_{i}$,  $U_{i}$) and continues to the next layer. 
 
 \begin{algorithm}[t]
    \caption{MILPBounds($N_\ell\circ \ldots \circ N_1$, $x$, $\epsilon$)}
    \label{algo::milpbounds}
    \DontPrintSemicolon
    $l=[]$; $u=[]$\;
    constraints = $\{\max(0,x_i-\epsilon)\leq z_{0,i}\leq \min(1,x_i+\epsilon) \mid i\in [d_{in}]\}$ \;\label{line:init}
    \For{$i=1; i\leq\ell; i++$}{
        
        \For{$m=1; m<m_i; m++$}{
        $l_{i,m}$ 
            = MILPSolve($\min b_{i,m} + \sum_{m' = 1}^{m_{i-1}} w_{i,m,m'} \cdot z_{i-1,m'}$ subject to constraints)\;\label{line:lower}
            $u_{i,m}$
            = MILPSolve($\max b_{i,m} + \sum_{m' = 1}^{m_{i-1}} w_{i,m,m'} \cdot z_{i-1,m'}$ subject to constraints)\;\label{line:upper}
            }
                constraints = constraints $\cup$ MIPVerify\_constraints($N_i$, $l_{i}$, $u_{i}$)\;\label{line:mipvr}
    
        }
    \Return{(l,u)}
\end{algorithm}

\sloppy
\paragraph{Batch verification}
We next describe how \tool forms a MILP for batch verification over the layers $N_{\ell+1},\ldots,N_L$. Given a batch $B$ and the bounds of $N_\ell$ for all inputs in $B$, \tool first defines the input space of $N_{\ell+1}$ as the union of the outputs of layer $N_\ell$. 
Then, it computes the real-valued bounds for every layer from $N_{\ell+1}$ to the output layer. 
Accordingly, it computes the same constraints as MIPVerify for these layers and adds the same constraint to look for an adversarial example. 
As in MIPVerify, if this MILP is infeasible, then there is no adversarial example, implying that \emph{all $\epsilon$-balls} of the inputs in the batch are robust. 
Otherwise,  some $\epsilon$-ball may be not robust. 

We next present our encoding for the inputs to $N_{\ell+1}$, which provides a simple way to identify the input in the batch whose $\epsilon$-ball may be non-robust.
The input to $N_{\ell+1}$ is the output of layer $N_\ell$.
A straightforward encoding is to bound each of its outputs in its minimal containing interval, i.e., $$\forall m.\ z_{\ell,m}\in [\min(\{bounds[x].l_{\ell,m}\mid x\in B\}),\max(\{bounds[x].u_{\ell,m}\mid x\in B\})].$$ However, this results in a very high overapproximation error and is also difficult to identify a good refinement if an adversarial example is detected.
Instead, we wish to encode a disjunction over the outputs of $N_{\ell}$, thereby forcing the input to $N_{\ell+1}$ to be contained in one of them: 
$$\bigvee_{x\in B} \bigwedge_{m\in [m_{\ell}]}\left(z_{{\ell},m}\geq bounds[x].l_{\ell,m}\wedge z_{{\ell},m}\leq bounds[x].u_{\ell,m}\right)$$
However, disjunctions are not directly expressible in MILPs.
Thus, we propose a MILP encoding adapting the big-M method for expressing the maximum function~\cite{BIG_M}.
Our encoding captures a function that takes a finite set of intervals and outputs a value in one of them.
Formally, given $k$ intervals $[l_1, u_1], \ldots , [l_k, u_k]$ such that $l_i \geq 0$ for every $i \in [k]$, our encoding introduces $k$ binary variables $I_1, \ldots , I_k \in \{0, 1\}$ and a real-valued variable $y$ for the output that is contained in one of the intervals. 
Our constraints force that: (1)~exactly one interval is picked (by requiring that the sum of the binary variables is one) and (2)~if $I_i=1$, 
then $y\in [l_i,u_i]$. This is encoded by two types of constraints, each has a copy for each of the $k$ intervals.
The first type of constraints forces $y \geq l_i$ in case $I_i=1$. The second type of constraints forces $y \leq u_i$ in case $I_i=1$.
We further denote the ``big-M'' as the maximum upper bound $u_M = \max (u_1, \ldots, u_k)$.
Our encoding is:
\begin{subequations}\label{logical_disjunction_equation}
    \begin{equation}\label{logical_disjunction_1}
        \begin{gathered}
            \sum_{i = 1}^{k} I_i = 1
        \end{gathered}
    \end{equation}
    \begin{equation}\label{logical_disjunction_2}
        \begin{gathered}
            \forall i \in [k]: y \geq l_i \cdot I_i
        \end{gathered}
    \end{equation}
    \begin{equation}\label{logical_disjunction_3}
        \begin{gathered}
            \forall i \in [k]: y \leq u_i \cdot I_i + u_M \cdot (1 - I_i)
        \end{gathered}
    \end{equation}
\end{subequations}

\begin{restatable}[]{theorem}{ftc}\label{thm:or}
    \Cref{logical_disjunction_equation}
    is feasible if and only if there exists $y \in [l_i, u_i]$ for some $i\in [k]$.
\end{restatable}
\ifthenelse{\EXTENDEDVER<0}{\Cref{sec:proof}}{\citet[Appendix B]{Tzour25}}
shows the proof.
\sloppy
\tool uses this encoding to bound the output of $N_\ell$.
It introduces $k=|B|$ binary variables $I_1,\ldots,I_k$, and then, for each $m\in [m_\ell]$, it adds the above constraints for $y=z_{\ell,m}$ and the intervals $[bounds[x].l_{\ell,m},bounds[x].u_{\ell,m}]$ for every $x\in B$.  
It also relies on real-valued bounds $L_{i,m}$ and $U_{i,m}$ for all $i\in \{\ell+1,\ldots,L\}$ and $m\in [m_i]$, which are computed as described before, by solving the optimizations $L_{i,m}=\min{z'_{i,m}}$ and $U_{i,m}=\max{z'_{i,m}}$ over all constraints of layers $\ell,\ldots,i-1$.
Overall, given a batch $B$ with $k$ inputs and their bounds $[l^j_{\ell,m}, u^j_{\ell,m}]$ for every $j\in[k],m\in [m_\ell]$, the batch verification is encoded by \texttt{MILP\_{SUF}}:
%
% %%% indexes legend %%%
% i - layer index
% j - sample in batch, k - number of samples in batch
% m - ReLU neuron index, m_i - number of relus in layer i
% z' - before relu, z - after relu

\begin{subequations}\label{batch_encoding}

    \begin{equation}\label{batch_encoding_1} % or booleans
        \begin{gathered}
 \forall j \in [k]:\ I_j \in \{0,1\}, \quad  \sum_{j = 1}^{k} I_j = 1
        \end{gathered}
    \end{equation}
    \begin{equation}\label{batch_encoding_2} % or upper bounds
        \begin{gathered}
            \forall j \in [k], \forall m \in [m_{\ell}]: \quad z_{\ell,m} \geq l^{j}_{\ell,m} \cdot I_j, \quad  z_{\ell,m} \leq u^{j}_{\ell,m} \cdot I_j + u_{M,m} \cdot (1 - I_j)
        \end{gathered}
    \end{equation}

    \begin{equation}\label{batch_encoding_4} % affine transformation
        \begin{gathered}
            \forall i > \ell, \forall m \in [m_i]: \quad  z'_{i,m} = b_{i,m} + \sum_{m' = 1}^{m_{i-1}} w_{i,m,m'} \cdot z_{i-1,m'}
        \end{gathered}
    \end{equation}

    \begin{equation}\label{batch_encoding_6} % relu - active
    \forall i > \ell, \forall m \in [m_i] 
    \begin{cases}
    z_{i,m} = z'_{i,m}            & L_{i,m} \geq 0 \\%[1ex]
    z_{i,m} = 0 & U_{i,m} \leq 0\\
     \!\begin{aligned}%[b]
      z_{i,m} \geq 0;\quad z_{i,m} \geq z'_{i,m}; \quad z_{i,m} \leq U_{i,m} \cdot a_{i,m}; \\
       z_{i,m} \leq z'_{i,m} - L_{i,m} \cdot (1 - a_{i,m}); \quad a_{i,m} \in \{0,1\}
    \end{aligned} 
     & \text{else}
  \end{cases}
    \end{equation}
    \begin{equation}\label{batch_encoding_7} % or booleans
        \begin{gathered}
          z_{L,c} \leq \max_{c' \neq c} z_{L,c'} 
        \end{gathered}
    \end{equation}

\end{subequations}
where the $\texttt{max}$ function encoding is defined formally in~\citet{MIPVERIFY}.
\Cref{algo::milpbatch} shows the generation of this MILP. It begins by generating the disjunction over the outputs of $N_{\ell}$ (\Cref{line:or}). Then, for every layer $i$, it iterates the neurons and, for each, computes a lower and an upper bound on the affine function (Lines~\ref{line:lower1}--\ref{line:upper1}) by calling a MILP solver with the constraints up to layer $i-1$. Then, it adds the constraints of layer $i$ (\Cref{batch_encoding_4,batch_encoding_6}) and continues to the next layer.

 \begin{algorithm}[t]
    \caption{MILPBatch($N_L\circ \ldots \circ N_\ell$, $\{[l^j_{\ell,m},u^j_{\ell,m}]\mid m\in [m_\ell],j\in [k]\}$, $c$)}
    \label{algo::milpbatch}

    $L=[]$; $U=[]$\;
    constraints = MILP\_OR$(\{[l^j_{\ell,m},u^j_{\ell,m}]\mid m\in [m_\ell],j\in [k]\})$ \tcp*{\Cref{batch_encoding_1,batch_encoding_2}}\label{line:or}
    \For{$i=\ell+1; i\leq L; i++$}{
        
        \For{$m=1; m<m_i; m++$}{
        $L_{i,m}$
            = MILPSolve($\min b_{i,m} + \sum_{m' = 1}^{m_{i-1}} w_{i,m,m'} \cdot z_{i-1,m'}$ subject to constraints)\;\label{line:lower1}
            $U_{i,m}$
            = MILPSolve($\max b_{i,m} + \sum_{m' = 1}^{m_{i-1}} w_{i,m,m'} \cdot z_{i-1,m'}$ subject to constraints)\;\label{line:upper1}
        }
        constraints = constraints $\cup$ MIPVerify\_constraints($N_i$, $L_i$, $U_i$)\tcp*{(\ref{batch_encoding_4}) and (\ref{batch_encoding_6})}\label{line:mipvr1}
    }
    \Return{
    $\texttt{constraints}\cup  \{z_{L,c} \leq \max_{c' \neq c} z_{L,c'}\}$}
\end{algorithm}

\paragraph{Refinement}
We next explain how \tool performs refinement, in case \texttt{MILP\_SUF} is feasible. 
This failure can arise either from a genuinely non-robust $\epsilon$-ball or from a spurious counterexample caused by excluding the constraints of the first layers $N_1,\ldots,N_\ell$.
If \texttt{MILP\_SUF} over a batch $B$ is feasible, by \Cref{thm:or}, the assignment identifies $x\in B$ (where $I_x=1$) and $y\in bounds[x][\ell]$ that is an adversarial example (i.e., $\text{argmax}(N_{L}\circ\ldots \circ N_{\ell+1}(y)) \neq c$). 
For this $x$, \tool performs refinement: it runs MIPVerify on the entire network $N$ and $B_\epsilon^\infty(x)$ (\Cref{line::refine}).
Note that refinement of $B$ to larger sets containing $x$ is not viable, because an adversarial example is found within the bounds of $x$, independently of the bounds of the other inputs in the batch (because of our disjunction encoding).
To expedite MIPVerify, \tool passes it the bounds for $x$ up to layer $\ell$.

\subsection{Constructing Batches by Hierarchical Clustering}\label{sec:cluster}
In this section, we describe our binary tree used for constructing the batches (in \Cref{line::search_batch}).

\paragraph{Goal} The goal of the binary tree is to provide an efficient approach to construct a batch in every iteration. To boost the batch verification, the inputs in the batch should have $\epsilon$-balls whose network's computations are as similar as possible.
This is because the closer the computations, the smaller the bounds of the last layers whose analysis is joined and the fewer ReLUs that become unstable in the batch.
However, identifying the $\epsilon$-balls with the closest computations requires analyzing each $\epsilon$-ball separately, defeating the purpose of the batch verification. 
Instead, as described before, we approximate their closeness by the Hamming distance of the activation patterns of the inputs at the center of the $\epsilon$-balls. 
  
\paragraph{Activation patterns}
The activation pattern (AP) of an input $x$ is a boolean vector whose size equals the number of ReLU neurons in the network. An entry $i$ in the vector is $1$ or $0$ depending on whether the respective neuron is active (i.e., its input is positive). Formally:
\begin{definition}[Activation Patterns]~\label{def:ap}
    Given a neural network $N$ with $n$ ReLU neurons and an input $x$, the \emph{activation pattern} of $x$ through $N$ is a boolean vector $r^x \in \{0, 1\}^n$, where $r^x_i$ is the state of the $i^\text{th}$ ReLU when $x$ fed into $N$.
    If the state is active (i.e., the input of the neuron $i$ is positive given $x$), $r^x_i=1$, otherwise, $r^x_i=0$.
\end{definition}

Unlike image similarity metrics, AP depends on the network. Thus, it captures the similarity of inputs with respect to how the network perceives them.
We rely on the \emph{Hamming distance} for measuring the distance of APs. 
Formally, the distance of two inputs is the number of different bits in their activation patterns: $dist(x,y)=|\{j \in [n]\mid r^x_j\neq r^y_j\}|$. 
The lower the Hamming distance of two APs over inputs $x$ and $y$, the fewer ReLU neurons that have distinct states. Consequently, the fewer unstable neurons that stem from grouping these inputs' $\epsilon$-balls and the lower the batch verification's complexity.

\paragraph{H-Cluster}
Given the activation patterns of the inputs in $S$, \tool clusters the inputs using \emph{Hierarchical Clustering (H-Cluster)} with the complete-linkage criteria.
H-Cluster greedily constructs a \emph{dendrogram}, a diagrammatic representation of the cluster hierarchy.
H-Cluster begins by forming 
a pairwise-distance matrix $D_{|S| \times |S|}$ of the inputs in $S$ by the Hamming distance of their AP (i.e., $D[x,y]=dist(x,y)$).
Then, H-Cluster builds the dendrogram bottom-up.
Initially, it forms a list of clusters, each contains one input.
At each step, it merges the clusters with the minimal distance
(follows by the complete linkage criteria).
The distance of two clusters is the maximal distance of any two inputs in the clusters: $dist(A,B) = \max_{x \in A, y \in B} dist(x,y)$. 
Figure \ref{fig::dendrogram} exemplifies a dendrogram.

\paragraph{Binary tree} Given the dendrogram over the activation patterns, \tool constructs a binary tree $\mathcal{T}$ over the inputs.
The tree enables it to construct a batch of up to a given size and remove a batch with logarithmic complexities in the size of the input set $S$, on average.
For every activation pattern $r^x$ in the dendrogram, \tool introduces a respective leaf node labeled by the input~$x$.
For every split in the dendrogram, \tool introduces an inner node, and the relation between the nodes follows exactly the structure of the dendrogram. Every inner node is marked by the number of leaves in its subtree. This enables \tool to easily construct a cluster up to a certain size.      
The number of nodes in $\mathcal{T}$ is $O(2|S|)$ and its depth ranges between $O(log_2(|S|))$ and $O(|S|)$.
The lower bound is obtained when at every iteration of H-Cluster, all clusters are merged with some cluster, resulting in a full and complete binary tree.
The upper bound is obtained when the first iteration of H-Cluster merges two inputs, and afterwards every iteration merges the largest cluster with a singleton cluster.
Figure \ref{fig::binary_tree} illustrates a binary tree constructed by \tool.
We next describe how \tool constructs a batch and how it removes a batch.
To form a batch of up to size $k$, \tool runs a pre-order traversal from the root of $\mathcal{T}$. When it reaches a node whose number of leaves is at most $k$, it forms a batch that consists of all its leaves (by continuing the pre-order traversal) and returns it.
Pruning a batch is obtained by removing the inner node that \tool used to construct the batch and updating the batch sizes in each node along the path back to the root.
Thus, this operation's average complexity is logarithmic in the tree size, which is $|S|$.

\subsection{Adaptive Selection of Batch Size via a Multi-Armed Bandit}\label{sec:mab}
In this section, we describe how \tool learns the optimal batch size (\Cref{line::recommend}).
This step provides another advantage of verification of a large set of $\epsilon$-balls: not only \tool can scale the analysis using mini-batch verification but also if the set $S$ is large it can dynamically learn the optimal mini-batch size. In particular, it may begin from larger mini-batches, for the relatively close inputs in $S$, and as the inputs become farther apart, it can dynamically reduce the mini-batch size. We next describe the mechanism for predicting the batch size via a multi-armed bandit (MAB) agent and how it leverages the verification of previous batches to predict the next batch size.

\paragraph{Multi-armed bandit} The multi-armed bandit (MAB) is a reinforcement learning problem where an agent iteratively selects an arm from a fixed set of arms $\{1,\ldots,K\}$. 
Each arm has an unknown distribution for reward.
After selecting an arm, a random reward is sampled from the arm's distribution and added to the agent's total reward. 
The agent's goal is to maximize their total reward. 
During the selection process, the agent learns the distributions of the arms. Consequently, they face the known \emph{exploration-exploitation} trade-off: at each iteration the agent can choose the arm with the highest expected reward (\emph{exploitation}) or an arm that may have better rewards (\emph{exploration}).
The Thompson Sampling is a strategy for selecting the arms that effectively balances exploration and exploitation~\cite{THOMPSON_SAMPLING}. 
Specifically, we focus on the Gaussian mean-variance bandits (MVTS) algorithm proposed by~\citet{THOMPSON_MVTS}, in which the arms' distributions are Gaussian.
In this case, the Thompson Sampling solves the Risk-Averse MAB problem, namely it balances reward maximization and variance minimization (via a risk tolerance factor $\rho$ added to the agent's goal).
In each iteration, the agent samples from the arms' distributions and selects the arm that optimizes the mean-variance objective function.
Given the reward, it updates the arms' distributions.

\begin{figure*}[t]
    \centering
    \includegraphics[width=0.9\linewidth, trim=160 370 190 0, clip, page=13]{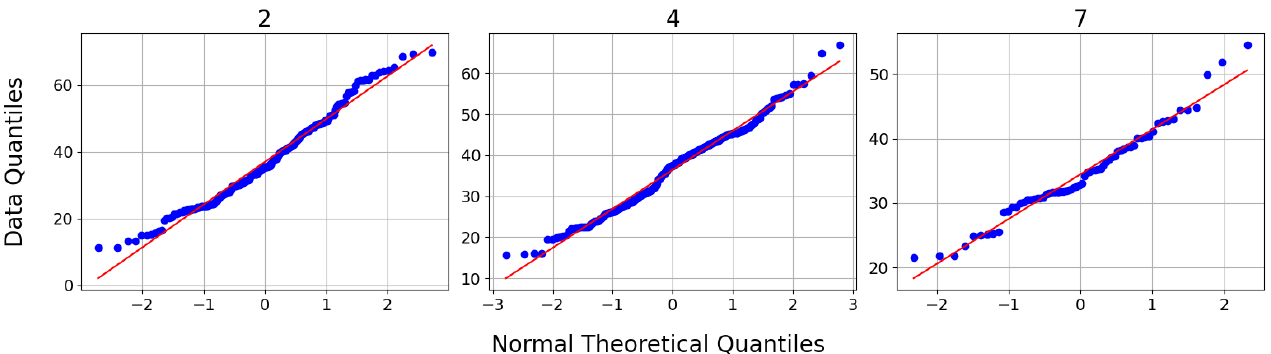} % left bottom right top
    \caption{A quantile-quantile plot comparing the empirical velocity distributions for batch sizes 2, 4, and 7 against the theoretical Gaussian distributions with the same mean and variance.}
    \label{fig::qq_plot_mab}
\end{figure*}

\paragraph{MAB for batch sizes} 
We define an arm for every mini-batch size up to \texttt{MAX\_BATCH\_SIZE}, which is a hyper-parameter. After the MAB agent selects an arm $k$, \tool constructs a batch of up to that size and verifies it. We define the reward as the velocity of this batch. 
Recall that velocity is distance divided by time. 
We define a batch $B$'s distance as the number of inputs in $B$
whose $\epsilon$-balls are analyzed jointly (i.e., the inputs whose robustness state is updated in~\Cref{line::robust_b}). 
The batch's time is the total runtime invested in proving the batch, i.e., the bound computation for all $x\in B$ from layer $N_1$ to layer $N_{\ell}$, plus the time of the batch verification. This time excludes the refinement time (i.e., the time spent to certify individual $\epsilon$-balls who failed during the batch verification), because it is independent of the batch effort, and we account for those failures in the distance calculation.
Our MAB leverages MVTS, since in practice the batch velocity of every batch size is approximately Gaussian-distributed.
\Cref{fig::qq_plot_mab} shows a quantile-quantile plot demonstrating that 
the velocity distributions are close to Gaussian, for different batch sizes, on an MNIST fully connected classifier with five hidden layers of 100 neurons each.

\paragraph{Unifying arms}
The higher the number of arms, the longer it takes for MAB to converge to the arms' distributions. 
To expedite its convergence, we partition the arms into \emph{buckets}. 
Each bucket consists of values $\{i,\ldots, i+\texttt{BUCKET\_SIZE}-1\}$. If an arm is selected, the maximal value in the bucket is used as the predicted batch size. If a batch of size $k$ is verified, its velocity is used for updating the distribution of the arm containing this value $k\in \{i,\ldots,i+\texttt{BUCKET\_SIZE}-1\}$. 
For example, for $\texttt{MAX\_BATCH\_SIZE}= 8$ and $\texttt{BUCKET\_SIZE} = 2$, there are four arms: $\{1, 2\}, \{3, 4\}, \{5, 6\}, \{7, 8\}$. If MAB selects the arm $\{5, 6\}$, it recommends using a batch size of $6$. If after this recommendation, \tool builds a batch of size $3$, this batch's velocity is used to update the distribution of the arm $\{3,4\}$.

\subsection{Choosing The Layer for Batch Verification}\label{sec:split}
We next explain how \tool chooses the layer $N_\ell$ for the batch verification.
Selecting $N_\ell$ is challenging due to the scalability-precision trade-off introduced by our batch verification, which is faster but adds overapproximation error. Although the batch's input layer $N_\ell$ does not add overapproximation error (due to our disjunction encoding), the following layers exhibit overapproximation error, because the real-valued bounds of every neuron consider all the batch's $\epsilon$-balls.  
The further the layer from $N_\ell$, the higher the overapproximation error.
This may suggest to favor $N_\ell$ closer to the output layer.
However, \tool analyzes each $\epsilon$-ball separately up to layer $N_\ell$.
Thus, the closer  $N_\ell$ to the output layer the lower the reduction in the overall analysis time compared to one by one verification.
This may suggest to favor $N_\ell$ closer to the input layer.

For convolutional neural networks, we choose $N_{\ell}$ as the last convolutional layer (before the fully connected layers). The motivation is that the output of this layer tends to be sufficiently discriminative across different classes.
Also, computing the bounds of convolutional layers is relatively fast, since their neurons get as input only part of the neurons in the previous layer.

For fully connected networks, we have not observed a single global layer that is effective for splitting.
As known, the goal of neurons in neural networks is to extract from previous neurons the information relevant for the classification. 
The better their extracted information the tighter the bounds.
For some networks,  the bounds at earlier layers are sufficiently tight for batch verification.
Thus, splitting in an early layer enables \tool to prove robustness. 
For other networks, splitting in an early layer leads to spurious counterexamples and triggers more refinements.
To estimate the best $N_\ell$, we rely on initialization via sampling~\cite{RANDOM_SEARCH_HYPERPARAMETER}. For each layer $l$, we estimate its effectiveness by sampling an input from $S$ and verifying its  $\epsilon$-ball when $\ell=l$. We define $\ell$ as the layer with the minimal analysis time.
Our estimate considers batches of size $k=1$ to eliminate the influence of the similarity of the inputs in the batch.
Our approach has several advantages. First, it is computationally efficient, since it focuses on batches of size one and relies on the verification of only $L-1$ $\epsilon$-balls. Second,  
it advances the task of group verification, since it determines the robustness status of the sampled inputs.  
Third, it does not rely on an offline mechanism or hyper-parameter tuning, which may not generalize well to an unseen network.
Our estimate approach is independent of \tool's analysis and can be improved by other mechanisms, e.g., online learning of automated reasoning strategies for a set of similar problems~\cite{SETS_ONLINE_LEARNING}. 

\begin{algorithm}[t]
    \caption{learnSplitLayer($N$, $S$, $c$, $\epsilon$, is\_robust)}
    \label{algo::split_learning}
    \DontPrintSemicolon

        $\text{layers\_times} = []$ \tcp*{Dictionary mapping layers to runtimes}
        \For{$l \in \{1, \ldots , L-1\}$} {
            $x$ = $\text{uniform}(S)$ \tcp*{An input sample}
            $\text{start\_time} = \text{current\_time}()$ \;
            $\text{bounds}$ = \text{MILPBounds}($N_l\circ\ldots \circ N_1$, $x$, $\epsilon$) \tcp*{Compute bounds up to $N_l$}\label{line::splitb}
            $\text{cex}$ = \text{MILPSolve}(\text{MILPBatch}($N_L\circ \ldots \circ N_{l+1}$, $\{\text{bounds}[l]\}$, $c$)) \tcp*{Verify the rest} \label{line::splite}
            \lIf(\tcp*[f]{Refine}){$\text{cex}\neq \bot$}{$\text{cex}$ = \text{MIPVerify}($N$, $x$, $\epsilon$, $\text{bounds}$)}\label{line::splitref}
            $\text{is\_robust}[x]$= $\text{cex} == \bot$? Robust : Non-Robust \;
            $\text{layers\_times}[l]$ = $\text{current\_time}() - \text{start\_time}$ \;
            $S = S \setminus \{x\}$
        }
        $\ell$ = $\argmin (\text{layers\_times})$ \tcp*{Choose the layer with the shortest runtime}
        \Return{$\ell$, $S$, is\_robust}

\end{algorithm}

\Cref{algo::split_learning} shows how \tool picks the layer to split. It takes as input the network $N$, the set of inputs $S$, the class $c$, the $\epsilon$, and the dictionary \texttt{is\_robust}.  It maintains a dictionary \texttt{layers\_times} mapping a layer to its analysis time for a single $\epsilon$-ball of a sampled input. 
For each layer, \tool samples an input $x$ and verifies its $\epsilon$-ball when splitting in this layer (\Cref{line::splitb}--\Cref{line::splite}). If a spurious counterexample is discovered, it analyzes the $\epsilon$-ball without splitting (\Cref{line::splitref}). 
Then, it updates the status in \texttt{is\_robust} and removes $x$ from $S$.
Lastly, it returns the layer with the minimal runtime.

\subsection{Complexity Analysis of \tool}\label{complexity}

In this section, we analyze the asymptotic complexity of our approach. We begin with the analysis time of a single batch verification, followed by the overall complexity analysis of \tool. 

\paragraph{Batch verification}
The verification of a batch $B$ of size $k$ includes (1)~computing the bounds of each input in $B$ up to layer $\ell$ (\Cref{line::prefix_boundsb}--\Cref{line::prefix_boundse}), (2)~the batch verification (\Cref{line::verify_batch_mid}), and (3)~the verification of $\epsilon$-balls suspected as not robust (\Cref{line::refine}).
The asymptotic time complexity of a MILP is exponential in the number of binary variables. In our setting, this number is the sum of the ReLU neurons, $k$ (for the disjunction,~\Cref{batch_encoding_1}) and $d_{out}-1$ (for checking if class $c$ might not have the maximal score,~\Cref{batch_encoding_7}).
The complexity is thus $T(B)=O(k \cdot 2^{\sum_{i=1}^{\ell} m_i} + (2^{k+d_{out}+\sum_{i=\ell+1}^L m_i})+r \cdot 2^{d_{out}+\sum_{i=1}^L m_i})$, where $m_i$ is the number of ReLU neurons in layer $i$ and $r$ is the number of $\epsilon$-balls suspected as non-robust.
We remind that verifying a batch after removing an $\epsilon$-ball that is suspected as non-robust (\Cref{line:addconssolve}) does not incur overhead.
For comparison, the asymptotic time complexity of verifying the $\epsilon$-balls one by one with MIPVerify is $O(k\cdot 2^{d_{out}+\sum_{i=1}^L m_i})$.
If $r=0$, \tool reduces MIPVerify's complexity by a factor of $k \cdot 2^{\sum_{i=1}^{\ell} m_i}$, which is added as an additive term, and multiplies by $2^k$ (which is independent on the network size).

\paragraph{\tool's complexity}
\tool begins by passing the inputs in $S$ through $N$, computing their activation patterns, and storing them in a dictionary.
The time complexity is negligible (compared to our analysis) and the memory complexity is  $O(|S|)$. 
The H-cluster incurs a runtime complexity of $O(|S|^3)$ and the resulting binary tree has a size of $O(2|S|)$.
The time complexity of operations on this tree is $O(log|S|)$ on average and $O(|S|)$ in the worst-case.
The time complexity of the operations on the MAB agent depends on the number of arms. Since it is a very small number, the time complexity is $O(1)$.
The dominant factor in \tool's runtime is the verification (known to be NP-hard~\cite{RELUPLEX}).
Let $B_1,\ldots,B_p$ be all batches in \tool's run, their time complexity is $\sum_{i=1}^p T(B_i)$.

\section{Evaluation}
\label{sec:evaluation}

In this section, we present the experimental results of our approach. We begin by discussing our implementation and evaluation setup. We then describe our experiments showing that: (1)~\tool expedites the approach of verifying local robustness $\epsilon$-ball by $\epsilon$-ball on average by 2.3x and up to 4.1x, (2)~\tool determines robustness for all inputs, whereas shared certification verification~\cite{SHARED_CERTIF_ARXIV} determines robustness only for 63\% of the inputs and shows a lower maximal speedup~(1.21x), (3)~the more inputs in the set, the higher the speedup in the analysis time per input: 100 inputs already enable a 2x speedup, and (4)~learning the optimal mini-batch size by our multi-armed bandit boosts the performance of \tool by 2.5x.

\paragraph{Implementation}
We implemented \tool in the Julia programming language (version 1.11.1), as a module wrapper for MIPVerify~\cite{MIPVERIFY}.
We extended MIPVerify to support our refinement, where some of the bounds are given and need not be computed.
We used Gurobi (version 12.0.1)~\cite{GUROBI} as the MILP solver.
For the batch verification (\Cref{line::verify_batch_mid,line:addconssolve}) and MIPVerify (\Cref{line::refine}),
we set Gurobi's MIPFocus flag to 1 to guide it to focus on finding a feasible solution rather than the optimal one, and we set its SolutionLimit flag to 1 so that Gurobi would terminate when finding the first feasible solution.
These adaptations fit our setting since our goal is to determine whether an $\epsilon$-ball is robust or not, which translates to determining whether these MILPs are feasible or not.

\paragraph{Evaluation setup}
We conducted our experiments on Ubuntu 20.04.2 LTS OS on a dual AMD EPYC 7742 64-Core Processor server with 1TB RAM and 128 cores.
We compared \tool to MIPVerify~\cite{MIPVERIFY} on which we build our MILP encoding. Given a set of inputs, MIPVerify verifies their $\epsilon$-balls one by one. Its analysis time is the total analysis time over all $\epsilon$-balls.
Gurobi is used by \tool and MIPVerify and it is parallelized over 8 threads and has the same flags' values.
We evaluated \tool on the MNIST dataset~\cite{MNIST}, consisting of 28$\times$28 grayscale images of handwritten digits, and the CIFAR-10 dataset~\cite{CIFAR}, consisting of 32$\times$32 RGB images representing ten classes of common animals and vehicles.
For MNIST,
we adopted four network architectures from~\citet{CONV_ARCHITECTURES}: the convolutional networks ConvSmall and ConvMed, and the fully connected (FC) networks 5$\times$100 and 6$\times$100. They contain 3,604, 5,704, 500, and 600 ReLU neurons, respectively.
ConvSmall has two convolutional layers with ReLU, a fully connected layer of 100 ReLU neurons, and a fully connected layer with ten neurons for assigning the scores to each class. ConvMed is similar to ConvSmall but has slightly different padding and its first fully connected has 1000 ReLU neurons. The 5$\times$100 has five hidden layers and the 6$\times$100 network has six hidden layers, each with 100 ReLU neurons.
We trained our MNIST models using the PGD adversarial defense~\cite{PGD} with a perturbation limit of $\epsilon = 0.2$, for the convolutional networks, and a limit of $\epsilon = 0.1$, for the fully connected networks. 
For CIFAR-10, we trained a ConvMed network with 7,144 ReLU neurons using PGD with $\epsilon=0.001$.
We incorporated techniques to improve generalization and stability, including $L_1$ regularization, adaptive learning rate scheduling, and Xavier uniform weight initialization.
Training was performed using Adam~\cite{ADAM}, for 6 epochs with a batch size of 128. Additionally, we executed standard scaling to transform pixel values for improved performance and added a corresponding normalization layer when loading the models.
The natural accuracy of the networks is 96\% for MNIST ConvMed, 93\% for ConvSmall and 5$\times$100, 91\% for 6$\times$100 and 48\% for CIFAR-10 ConvMed (similar to the accuracies of the networks evaluated by~\citet{SHARED_CERTIF_ARXIV}).
For the convolutional networks, the split layer $\ell$ is the last convolutional layer. For the fully connected networks, $\ell$ is chosen by \texttt{learnSplitLayer}.
The maximum batch size is $\texttt{MAX\_BATCH\_SIZE} = 4$ for the ConvMed networks and $\texttt{MAX\_BATCH\_SIZE} = 8$ for the rest. 
The bucket size for unifying batch sizes is $\texttt{BUCKET\_SIZE} = 2$ and our MAB's $\rho$ is $100$.

\begin{table}[t]
    
    \begin{center}
    \caption{\tool vs. MIPVerify over different networks and $\epsilon$ on sets with 100 inputs of the same class.}
    \begin{tabular}{llcccccc}
        \toprule
        Dataset & Network & $\epsilon$ & $c$ &  Cert.  & MIPVerify & \tool & Speedup \\
                &        &           &       & Rate    & [hours]   & [hours] &       \\
        
        \midrule
        MNIST   & ConvMed  &  0.03 & 0 & $99 / 99$ & 27.22 & 8.00 & 3.4 \\
                &           &      & 1 & $100 / 100$ & 36.42 & 14.08 & 2.6 \\
                &           &      & 2 & $96 / 98$ & 27.61 & 9.99 & 2.7 \\
                &           &      & 3 & $95 / 98$ & 27.08 & 9.82 & 2.7 \\
                & ConvSmall & 0.05 & 0 & $95 / 98$ & 2.04 & 1.36 & 1.5 \\
                &           &      & 1 & $98 / 100$ & 2.30 & 0.96 & 2.4 \\
                &           &      & 2 & $78 / 91$ & 1.90 & 1.33 & 1.4 \\
                &           &      & 3 & $78 / 89$ & 2.07 & 1.38 & 1.5 \\
                & 5$\times$100 & 0.03 & 0 & $96 / 98$ & 13.87 & 4.92 & 2.8 \\
                &              &      & 1 & $96 / 100$ & 23.63 & 5.75 & 4.1 \\
                &              &      & 2 & $79 / 90$ & 8.80 & 5.70 & 1.5 \\
                &              &      & 3 & $78 / 93$ & 10.88 & 6.23 & 1.7 \\
                & 6$\times$100 & 0.03 & 0 & $88 / 97$ & 11.50 & 5.46 & 2.1 \\
                &              &      & 1 & $92 / 100$ & 33.85 & 12.79 & 2.6 \\
                &              &      & 2 & $75 / 88$ & 15.10 & 8.77 & 1.7 \\
                &              &      & 3 & $94 / 98$ & 14.37 & 5.43 & 2.6 \\
        \midrule
        CIFAR-10  & ConvMed & 0.001 & 0 & $37 / 39$ & 2.64 & 1.07 & 2.4 \\
                  &         &       & 1 & $61 / 62$ & 4.18 & 1.97 & 2.1 \\
                  &         &       & 2 & $31 / 31$ & 1.74 & 0.87 & 2.0 \\
                  &         &       & 3 & $39 / 41$ & 2.89 & 1.60 & 1.8 \\
        
        \bottomrule        
    \end{tabular}
        \label{tab:results_per_label}
    \quad
\end{center}
\end{table}
\begin{figure*}[t]
    \centering
    \includegraphics[width=0.85\linewidth, trim=0 270 310 0, clip, page=18]{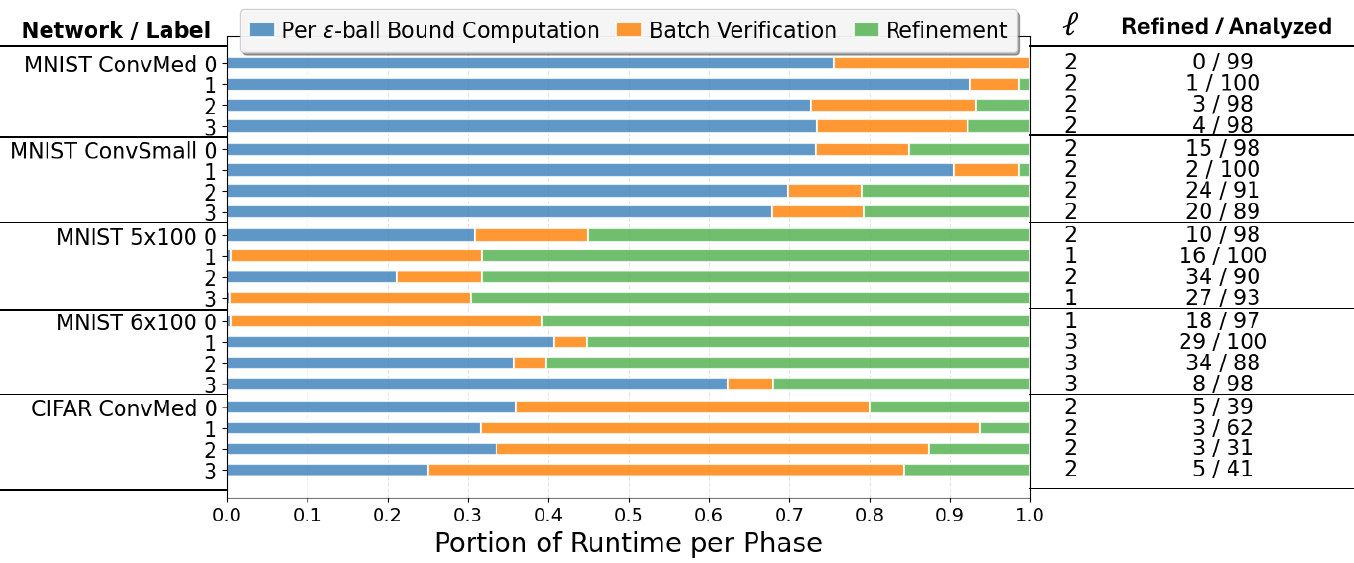} % left bottom right top
    \caption{Runtime breakdown and refinement frequency for the experiments in~\Cref{tab:results_per_label}. }
    \label{fig::breakdown}
\end{figure*}

\subsection{Performance Analysis}
In this section, we evaluate \tool's effectiveness in group verification and compare to MIPVerify.

\paragraph{Group verification on large sets}
We begin with an experiment on large sets of inputs. 
In this experiment, we consider all networks. For each, we run \tool on sets $S$ of the first 100 inputs of the same class (for several classes) and different values of $\epsilon$. 
We compare its analysis time to MIPVerify's analysis time.
We remind that both verifiers are complete (i.e., correctly determine whether an $\epsilon$-ball is robust or not).
\Cref{tab:results_per_label} shows the certification rate, the total analysis time of both approaches and \tool's speedup.
The certification rate is the number of inputs whose $\epsilon$-balls are verified as robust (by both approaches) divided by the number of correctly classified inputs.
Results show that \tool's speedup is 2.3x on average and up to 4.1x.
The highest speedup is obtained on MNIST 5$\times$100, where \tool reduces the analysis time from 24 hours to 6 hours.
\Cref{fig::breakdown} shows the runtime breakdown and refinement frequency.
It shows that most $\epsilon$-balls are verified within a batch. 
The MNIST convolutional networks have the shortest refinement phase, while the fully connected networks have the shortest bound computation and longest refinement.

\begin{table}[t]
    \small
    \begin{center}
    \caption{\tool vs. MIPVerify over different networks and $\epsilon$ on sets with the first 100 test set images.}
    \begin{tabular}{llccccc}
        \toprule
        Dataset & Network & $\epsilon$ & Cert. & MIPVerify & \tool & Speedup  \\
                &         &            & Rate  & [hours]   & [hours] &        \\
        
        \midrule
        MNIST   & ConvMed  & 0.03 & $98 / 100$ & 31.98 & 11.96 & 2.6 \\
                & ConvSmall & 0.1 & $69 / 88$ & 27.86 & 19.33 & 1.4 \\
                & 5$\times$100 & 0.03 & $84 / 92$ & 11.71 & 4.85 & 2.4 \\
                & 6$\times$100 & 0.03 & $80 / 92$ & 18.85 & 11.80 & 1.6 \\
        \midrule
        CIFAR-10  & ConvMed &  0.001 & $49 / 50$ & 4.03 & 2.43 & 1.6 \\
        
        \bottomrule        
    \end{tabular}
        \label{tab:results_entire_dataset}
    \quad
\end{center}
\end{table}

\paragraph{Group verification on small sets}
Next, we evaluate \tool on small sets of inputs. 
We consider different networks and $\epsilon$ values. For each, we ran each approach on the first 100 test set images, consisting of different classes (\tool ran separately on each class).
For such sets, \tool almost does not benefit from learning the optimal mini-batch size and does not benefit from clustering inputs with similar computations of the networks.
Thus, this experiment is challenging for \tool.
\Cref{tab:results_entire_dataset} shows the total analysis time (over all classes), for \tool and MIPVerify.
It further shows the certification rate and \tool's speedup.  
\tool achieves up to a 2.6x speedup compared to MIPVerify. MNIST ConvMed has the best speedup (and the highest certification rate, $98 / 100$).

\paragraph{Shared certification}
We next discuss the empirical differences between \tool and shared certification~\cite{SHARED_CERTIF_ARXIV}.
Shared certification relies on preprocessing (which takes multiple hours) for generating templates which can expedite the analysis on unseen $\epsilon$-balls. It focuses on incomplete verification, i.e., the robustness of $\epsilon$-balls can remain unknown.
For example, in our experiment over sets with the first 100 test set images, for ConvSmall and $\epsilon = 0.1$,
\tool precisely determines the certification rate, which is $69/88$ (78\%).
However, verification that relies on the DeepZ abstract domain~\cite{DEEPZ}, like shared certification,
can only prove that 56 $\epsilon$-balls are robust (certification rate of 63\%).
Further, the highest speedup of shared certification over DeepZ (on which it builds) is 1.21x, while \tool's speedup over MIPVerify is at least 1.4x and up to 2.6x.

\paragraph{Analysis time per $\epsilon$-ball}

\begin{figure*}[t]
    \centering
    \includegraphics[width=0.9\linewidth, trim=150 350 180 0, clip, page=12]{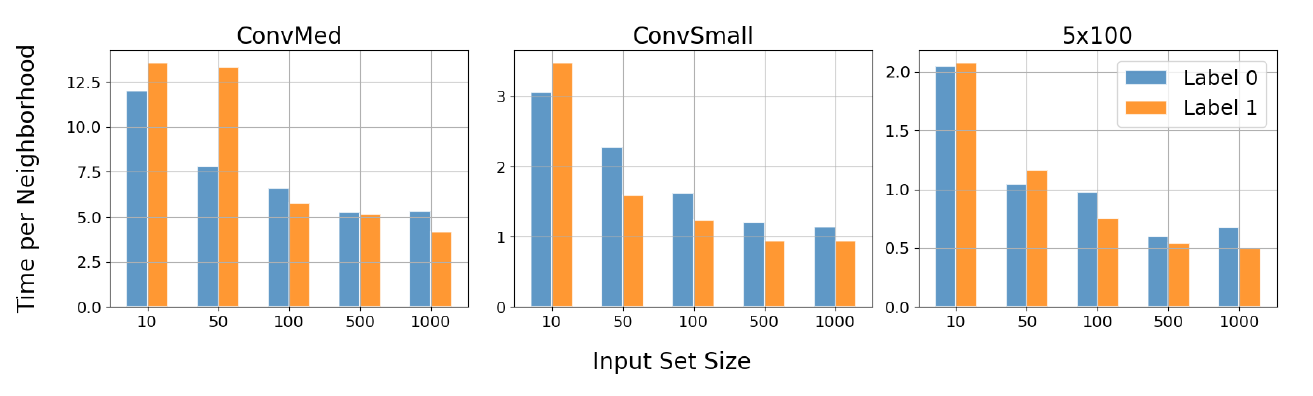} % left bottom right top
    \caption{Verification time per $\epsilon$-ball (in seconds) for different set sizes and the first two classes of MNIST.}
    \label{fig::larger_sets}
\end{figure*}
Next, we show that as the input set size increases, the average analysis time per $\epsilon$-ball decreases. 
This experiment shows the effectiveness of clustering inputs in batches, especially for larger sets where the likelihood of finding inputs with similar network computations increases. Additionally, the larger the input set, the better our MAB agent in predicting optimal mini-batch sizes.
In this experiment, we focus on the MNIST networks and $\epsilon = 0.00001$. We consider the first two classes of MNIST and for each we construct input sets of different sizes from 10 to 1000 (for class 0, up to 980, which is its test set's size).
\Cref{fig::larger_sets} shows the average analysis time per $\epsilon$-ball as a function of the size of the set.
Results show that, compared to the set with 10 inputs, \tool obtains up to a 4x speedup for 5$\times$100, up to a 3.5x speedup for ConvSmall and up to a 3.3x speedup for ConvMed, for the set with 1000 inputs.
In fact, for all networks, a 2x speedup is observed already for 100 inputs.
This shows the effectiveness of our batches and learning.

\begin{table}[t]
    \small
    \begin{center}
    \caption{\tool vs. a variant that randomly chooses the mini-batch sizes. The times of the random variant are averaged over three repetitions.
    The input set is the full test set of MNIST's first class.}
    \begin{tabular}{llcccc}
        \toprule
        Dataset & Network & $\epsilon$ & \tool [m] & \tool w/ random size [m] & Speedup\\

        \midrule
        MNIST & ConvMed  & 0.00001 & 142.55 & 639.73 & 4.5 \\
              & ConvSmall & 0.00001 & 35.79 & 53.64 & 1.5 \\
              & 5$\times$100 & 0.00001 & 11.02 & 18.15 & 1.6 \\
        \bottomrule
    \end{tabular}
    \label{tab:mab_random}
    \quad
\end{center}
\end{table}

\subsection{Ablation Study}
In this section, we show the effectiveness of \tool's components.

\paragraph{MAB effectiveness}
We study the importance of learning the mini-batch size using a multi-armed bandit (MAB).
We consider a variant that randomly selects the mini-batch sizes.
We set  $\texttt{MAX\_BATCH\_SIZE}=16$ to make the learning more challenging for our MAB agent.
We run both approaches on the MNIST classifiers, where the input set is 980 test inputs of class 0.
\Cref{tab:mab_random} shows the analysis time and the speedup of \tool.
Results show that our MAB accelerates \tool's verification time by 4.5x for ConvMed, by 1.5x for ConvSmall and by 1.6x for 5$\times$100.

\begin{table}[t]
    \small
    \begin{center}
    \caption{\tool vs. a variant with a given split layer $\ell$, on sets of 50 MNIST images of the same class.}
    \begin{tabular}{lll ccccccc}
        \toprule
        Network & $\epsilon$ & $c$ & \tool [h] ($\ell$) & $\ell=1$ [h] & $\ell=2$ [h]& $\ell=3$ [h]& $\ell=4$ [h] & $\ell=5$ [h] \\
        \midrule
            5$\times$100 & 0.03 & 0 & 2.98 (1) & \textbf{2.75} & 2.94 & 3.60 & 4.62 & N/A \\ % predicted: 1, actual: 1, overhead 0.23
                         &      & 1 & 2.92 (1) & \textbf{2.71} & 4.34 & 5.36 & 7.03 & N/A \\ % predicted: 1, actual: 1, overhead 0.21
                         &      & 2 & 2.93 (1) & \textbf{2.89} & 3.33 & 3.35 & 3.63 & N/A \\ % predicted: 1, actual: 1, overhead 0.04
            6$\times$100 & 0.03 & 0 & 2.83 (1) & 2.75 & \textbf{2.72} & 3.17 & 3.59 & 4.35 \\ % predicted: 1, actual: 2, overhead 0.11
                         &      & 1 & 6.28 (3) & 7.99 & 7.16 & \textbf{6.23} & 8.56 & 10.69 \\ % predicted: 3, actual: 3, overhead 0.05
                         &      & 2 & 5.08 (1) & 5.02 & 4.76 & \textbf{4.41} & 4.59 & 5.38 \\ % predicted: 1, actual: 3, overhead 0.67

        \bottomrule
    \end{tabular}
    \label{tab:split_learning}
    \quad
\end{center}
\end{table}

\paragraph{Split layer}
We next show the effectiveness of \tool in learning the split layer $N_\ell$ (\Cref{algo::split_learning}).
We compare  to a variant that fixes $\ell$.
We evaluate on MNIST 5$\times$100 and 6$\times$100 and three input sets, each containing 50 images of the same class (0, 1 or 2).
\Cref{tab:split_learning} shows the verification time in hours and the layer that \tool chose.
The results show the importance of selecting a good $\ell$ and that the optimal $\ell$ varies between the networks and input sets.
The results also show that \tool selects the optimal split layer for most networks and classes with a overhead of 7.2 minutes, on average. Even when a sub-optimal $\ell$ is selected, the overhead does not exceed 40.2 minutes.

\begin{figure*}[t]
    \small
    \centering
    \begin{minipage}{0.48\linewidth}
        \begin{subtable}{\linewidth}
            \centering
            \caption{Analysis time on CIFAR-10 ConvMed, $\epsilon=0.001$.}
            \label{tab::similarity}
            \begin{tabular}{l ccc}
                \toprule
                $c$ & AP [m] & SSIM [m] & LPIPS [m] \\
                \midrule
                0  & 64.3 & 72.3 & 69.3 \\
                1  & 118.6 & 135.3 & 154.8 \\
                2  & 52.2 & 57.4 & 58.2 \\
                \bottomrule
            \end{tabular}
            
        \end{subtable}
        
       % \vspace{1em} % space between the two tables

        \begin{subtable}{\linewidth}
            \centering
            \caption{\tool without disjunction encoding, $\ell=0$, and a 10 hour limit on MNIST ConvSmall, $\epsilon=0.05$.}
            \label{tab::abstraction}
            \begin{tabular}{l cc}
                \toprule
                Similarity metric & \#Analyzed & \#Robust \\
                \midrule
                AP & 75 & 1 \\
                SSIM & 74 & 2 \\
                LPIPS & 76 & 2 \\
                \bottomrule
            \end{tabular}
        \end{subtable}
        \caption{\tool with different similarity metrics.}
    \end{minipage}
    \hspace{1em}
    \begin{minipage}{0.48\linewidth}
        \centering
        \includegraphics[width=0.9\linewidth, trim=0 120 470 0, clip, page=16]{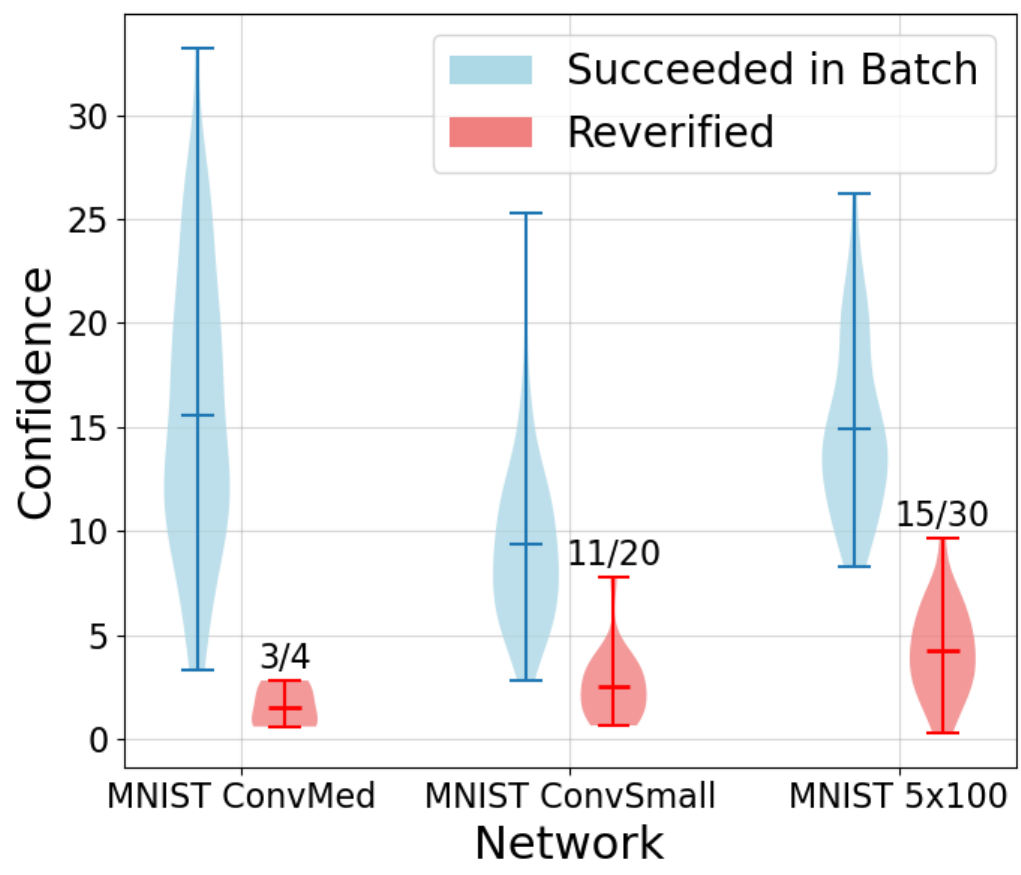}
        \captionof{figure}{The confidence of $\epsilon$-balls verified within a batch and $\epsilon$-balls that required a separate analysis.}
        \label{fig::confidence}
    \end{minipage}
\end{figure*}

\paragraph{Similarity metrics}
We next show the effectiveness of the activation pattern similarity.
We compare \tool to variants that replace the activation pattern similarity with SSIM~\cite{SSIM} and LPIPS~\cite{LPIPS} (using AlexNet activations).
We evaluate on CIFAR-10 ConvMed with an input set containing 100 inputs of the same class, for different classes, and $\epsilon=0.001$. 
\Cref{tab::similarity} shows the analysis time. It shows that the activation pattern similarity is better by 1.13x, on average.

\paragraph{Disjunction}
We next show the importance of our disjunction encoding (\Cref{logical_disjunction_equation}).
We evaluate a variant that abstracts $\epsilon$-balls in the input layer ($\ell=0$) with the minimal bounding box.
It does not refine $\epsilon$-balls, because it cannot easily identify a suspect non-robust $\epsilon$-ball.
We ran this variant with the three similarity metrics, on MNIST ConvSmall, $\epsilon=0.05$, and an input set consisting of 100 images of class 0, and 
a 10 hour limit. \Cref{tab::abstraction} reports the number of analyzed inputs (whose $\epsilon$-ball is included in a batch) and the number of $\epsilon$-balls that were proven robust.
On average, 25\% $\epsilon$-balls could not be analyzed within 10 hours, and at most two $\epsilon$-balls were proven robust. 
In contrast, \tool determines robustness for all 98 $\epsilon$-balls within 1.36 hours (\Cref{tab:results_per_label}).

\paragraph{Error case analysis}
We next show an empirical difference between $\epsilon$-balls that were proven within a batch and those that were refined.
The difference is the network's \emph{confidence} in their central input, i.e., the gap between the highest and second highest scores. %: $C(N,x,c)= N(x)_c - \max_{c'\neq c} N(x)_{c'}$.   
\Cref{fig::confidence} shows a violin plot over the confidence distribution of MNIST inputs of class $3$, for different networks, for $\epsilon$-balls that were proven in a batch (in blue) and refined (in red).
Above the red violin, we show the number of non-robust $\epsilon$-balls and the number of refined (suspected non-robust) $\epsilon$-balls. 
The figure shows that the confidence is lower for refined $\epsilon$-balls and that at least half of them are indeed non-robust.

\begin{figure*}[t]
    \centering
    \includegraphics[width=0.8\linewidth, trim=0 344 350 0, clip, page=17]{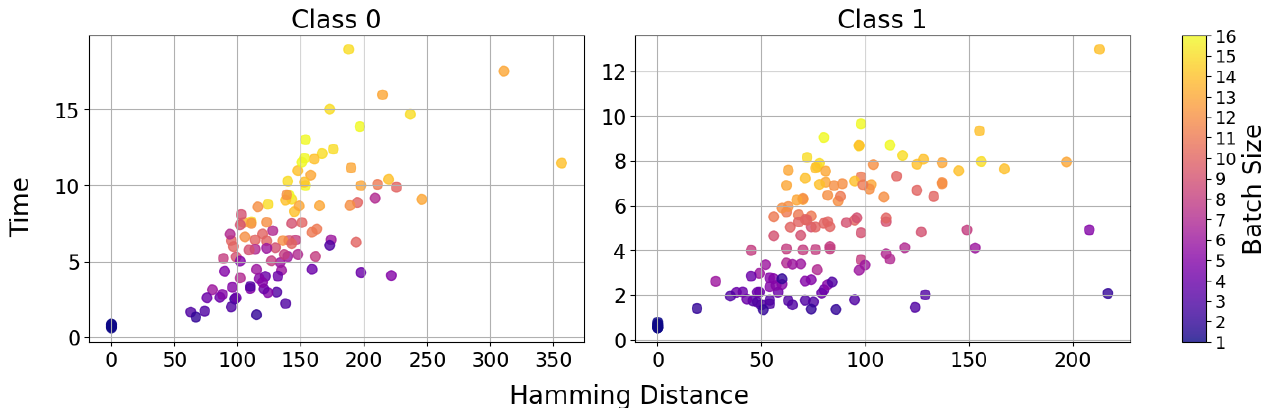} % left bottom right top
    \caption{The analysis time of \tool (in minutes) for different batches on MNIST ConvSmall, $\epsilon=0.05$.}
    \label{fig::effects_time}
\end{figure*}

\paragraph{Batch size vs. Hamming distance}
\tool constructs a batch by determining its size, with the MAB agent, and then choosing the most similar inputs. 
An alternative is to choose the most similar inputs, with a MAB agent that selects the maximal pairwise distance, and then determine the batch size.
We next show that this alternative is ineffective since the analysis time is not correlated to the maximal pairwise Hamming distance. 
We consider a variant that randomly selects a mini-batch size (up to 16) and then constructs a batch of exactly this size (if possible), to draw balanced statistics for all sizes. 
We evaluate on MNIST ConvSmall, $\epsilon=0.05$, and input sets consisting of all 980 and 1135 test set inputs of class 0 and 1.
\Cref{fig::effects_time} shows the analysis time of each batch as a function of the maximal pairwise Hamming distance. Each batch is colored by its size.  
The results show that the analysis time of every distance has a wide range, whereas batches of the same size have similar analysis time.
For example, for class 1, the analysis of batches with distance 70 takes 1-9 minutes, whereas for batches of size 8 it takes 4-6 minutes.
That is, the analysis time is related to the batch size (we remind that batches are evaluated by their velocity, which depends on the analysis time).

\section{Related Work}
\label{sec:related_work}

In this section, we discuss related work for boosting neural network verification.

\paragraph{Boosting by previous analysis}
Several neural network verifiers boost their analysis  by relying on previous analysis.
Verifiers targeting continuous verification (incremental verification) leverage the intermediate analysis results of a network to expedite the analysis of its variants obtained by further training or fine-tuning~\cite{CONTINUOUS, ONLINE_VERIF,IVAN}. 
FANC proposes \emph{proof sharing} for generating intermediate layer templates that capture the property being verified and adapt them for similar networks to expedite their verification~\cite{FANC}.
Shared certifications extend this concept to boost the analysis on unseen $\epsilon$-balls by generating abstract templates of intermediate analysis results during preprocessing~\cite{SHARED_CERTIF}.
DeepAbstract analyzes a set of inputs during preprocessing to identify similar neurons by their activation patterns~\cite{DEEPABSTRACT}. These neurons are abstracted when verifying $\epsilon$-balls, thereby boosting their analysis.  
Similarly, AccMILP analyzes a set of inputs during preprocessing to identify neurons with less impact on the network's accuracy~\cite{ACCMILP}. Accordingly, the verification of an $\epsilon$-ball performs linear relaxation only to these neurons.

\paragraph{Boosting by learning}
Several neural network verifiers employ learning to expedite their analysis.
\citet{TARGET_LABELS} boost existing verifiers by learning a prioritization over classes by their likelihood to be the classes of adversarial examples.
VeeP expedites local robustness verification of a network classifier in semantic feature neighborhoods by relying on active learning to partition the verification process into smaller steps~\cite{VEEP}. Similarly to \tool, it predicts the next step by computing the analyzer's velocity on previous steps.
Learning has also been proposed to expedite the analysis of verifiers in other domains.
\citet{MDP_VERIF} rely on machine learning to improve the verification of Markov decision processes (MDPs), efficiently analyzing probabilistic reachability and temporal properties without exhaustively exploring the entire state space.
\citet{MAB_HW_VERIF} reduce hardware verification efforts using a multi-armed bandit to automatically select the most promising test sequences.
\citet{SETS_ONLINE_LEARNING} learn optimal strategies for expediting the automated reasoning of a set of similar problems. They sample problem candidates, train a cost model to predict the runtime of a strategy for a given problem, and use it to dynamically select efficient strategies for future problems.

\section{Conclusion}
\label{sec:conclusions_and_discussion}

We present \tool, a verifier that analyzes the robustness of a neural network in a set of $L_\infty$ $\epsilon$-balls. 
\tool is sound and complete.
The key idea is to identify the $\epsilon$-balls for which the network has similar computations and group their analysis in a mini-batch. 
\tool relies on a multi-armed bandit to predict the optimal mini-batch size.
\tool begins the batch analysis in a middle layer and encodes the inputs to this layer precisely. 
This encoding also enables an effective refinement.
We evaluated \tool on fully connected and convolutional networks for MNIST and CIFAR-10. Experimental results show that \tool scales the verification on average by 2.3x and up to 4.1x, reducing verification time from 13 hours to 5 hours, on average.
Our results also show the importance of learning the optimal batch sizes: it scales \tool by 2.5x.

\section*{Acknowledgements}
We thank Yuval Shapira for his invaluable insights and the anonymous reviewers for their feedback.
This research was supported by the Israel Science Foundation (grant No. 2607/25).

\section*{Data-Availability Statement}
The data and source code that support the findings of this study are openly available. 
The reproducible artifact associated with this paper is described in \citet{BaVerLy-artifact}, and the implementation and experimental scripts can be accessed and reused at \url{https://github.com/Saarts21/BaVerLy}.

\bibliography{bib}


\begin{thebibliography}{76}

%%% the Web address.

\ifx \showCODEN    \undefined \def \showCODEN     #1{\unskip}     \fi
\ifx \showDOI      \undefined \def \showDOI       #1{#1}\fi
\ifx \showISBNx    \undefined \def \showISBNx     #1{\unskip}     \fi
\ifx \showISBNxiii \undefined \def \showISBNxiii  #1{\unskip}     \fi
\ifx \showISSN     \undefined \def \showISSN      #1{\unskip}     \fi
\ifx \showLCCN     \undefined \def \showLCCN      #1{\unskip}     \fi
\ifx \shownote     \undefined \def \shownote      #1{#1}          \fi
\ifx \showarticletitle \undefined \def \showarticletitle #1{#1}   \fi
\ifx \showURL      \undefined \def \showURL       {\relax}        \fi
% invisible to TeX
\providecommand\bibfield[2]{#2}
\providecommand\bibinfo[2]{#2}
\providecommand\natexlab[1]{#1}
\providecommand\showeprint[2][]{arXiv:#2}

\bibitem[Anthimopoulos et~al\mbox{.}(2016)]%
        {LUNG_PATTERN}
\bibfield{author}{\bibinfo{person}{Marios Anthimopoulos},
  \bibinfo{person}{Stergios Christodoulidis}, \bibinfo{person}{Lukas Ebner},
  \bibinfo{person}{Andreas Christe}, {and} \bibinfo{person}{Stavroula
  Mougiakakou}.} \bibinfo{year}{2016}\natexlab{}.
\newblock \showarticletitle{Lung Pattern Classification for Interstitial Lung
  Diseases Using a Deep Convolutional Neural Network}.
\newblock \bibinfo{journal}{\emph{IEEE Transactions on Medical Imaging}}
  \bibinfo{volume}{35}, \bibinfo{number}{5} (\bibinfo{year}{2016}),
  \bibinfo{pages}{1207--1216}.
\newblock
\urldef\tempurl%
\url{https://doi.org/10.1109/TMI.2016.2535865}
\showDOI{\tempurl}


\bibitem[Ashok et~al\mbox{.}(2020)]%
        {DEEPABSTRACT}
\bibfield{author}{\bibinfo{person}{Pranav Ashok}, \bibinfo{person}{Vahid
  Hashemi}, \bibinfo{person}{Jan K{\v{r}}et{\'i}nsk{\'y}}, {and}
  \bibinfo{person}{Stefanie Mohr}.} \bibinfo{year}{2020}\natexlab{}.
\newblock \showarticletitle{DeepAbstract: Neural Network Abstraction for
  Accelerating Verification}. In \bibinfo{booktitle}{\emph{Automated Technology
  for Verification and Analysis}}, \bibfield{editor}{\bibinfo{person}{Dang~Van
  Hung} {and} \bibinfo{person}{Oleg Sokolsky}} (Eds.).
  \bibinfo{publisher}{Springer International Publishing},
  \bibinfo{address}{Cham}, \bibinfo{pages}{92--107}.
\newblock
\showISBNx{978-3-030-59152-6}
\urldef\tempurl%
\url{https://doi.org/10.1007/978-3-030-59152-6\_5}
\showDOI{\tempurl}


\bibitem[Ayachi et~al\mbox{.}(2020)]%
        {AUTONOMOUS_DRIVING3}
\bibfield{author}{\bibinfo{person}{Riadh Ayachi}, \bibinfo{person}{Mouna Afif},
  \bibinfo{person}{Yahia Said}, {and} \bibinfo{person}{Mohamed Atri}.}
  \bibinfo{year}{2020}\natexlab{}.
\newblock \showarticletitle{Traffic Signs Detection for Real-World Application
  of an Advanced Driving Assisting System Using Deep Learning}.
\newblock \bibinfo{journal}{\emph{Neural Processing Letters}}
  \bibinfo{volume}{51} (\bibinfo{date}{02} \bibinfo{year}{2020}).
\newblock
\urldef\tempurl%
\url{https://doi.org/10.1007/s11063-019-10115-8}
\showDOI{\tempurl}


\bibitem[Bachute and Subhedar(2021)]%
        {AUTONOMOUS_DRIVING1}
\bibfield{author}{\bibinfo{person}{Mrinal~R. Bachute} {and}
  \bibinfo{person}{Javed~M. Subhedar}.} \bibinfo{year}{2021}\natexlab{}.
\newblock \showarticletitle{Autonomous Driving Architectures: Insights of
  Machine Learning and Deep Learning Algorithms}.
\newblock \bibinfo{journal}{\emph{Machine Learning with Applications}}
  \bibinfo{volume}{6} (\bibinfo{year}{2021}), \bibinfo{pages}{100164}.
\newblock
\showISSN{2666-8270}
\urldef\tempurl%
\url{https://doi.org/10.1016/j.mlwa.2021.100164}
\showDOI{\tempurl}


\bibitem[Bahdanau et~al\mbox{.}(2015)]%
        {MACHINE_TRANSLATION}
\bibfield{author}{\bibinfo{person}{Dzmitry Bahdanau},
  \bibinfo{person}{Kyunghyun Cho}, {and} \bibinfo{person}{Yoshua Bengio}.}
  \bibinfo{year}{2015}\natexlab{}.
\newblock \showarticletitle{Neural Machine Translation by Jointly Learning to
  Align and Translate}. In \bibinfo{booktitle}{\emph{3rd International
  Conference on Learning Representations, {ICLR} 2015, San Diego, CA, USA, May
  7-9, 2015, Conference Track Proceedings}},
  \bibfield{editor}{\bibinfo{person}{Yoshua Bengio} {and} \bibinfo{person}{Yann
  LeCun}} (Eds.).
\newblock
\urldef\tempurl%
\url{https://doi.org/10.48550/arXiv.1409.0473}
\showDOI{\tempurl}


\bibitem[Bak and Duggirala(2017)]%
        {STAR_DOMAIN1}
\bibfield{author}{\bibinfo{person}{Stanley Bak} {and} \bibinfo{person}{Parasara
  Duggirala}.} \bibinfo{year}{2017}\natexlab{}.
\newblock \showarticletitle{Simulation-Equivalent Reachability of Large Linear
  Systems with Inputs}. \bibinfo{pages}{401--420}.
\newblock
\showISBNx{978-3-319-63386-2}
\urldef\tempurl%
\url{https://doi.org/10.1007/978-3-319-63387-9_20}
\showDOI{\tempurl}


\bibitem[Balunovic et~al\mbox{.}(2019)]%
        {GEOMETRIC_PERTURBATIONS_SHARED_CERTIFICATES}
\bibfield{author}{\bibinfo{person}{Mislav Balunovic},
  \bibinfo{person}{Maximilian Baader}, \bibinfo{person}{Gagandeep Singh},
  \bibinfo{person}{Timon Gehr}, {and} \bibinfo{person}{Martin Vechev}.}
  \bibinfo{year}{2019}\natexlab{}.
\newblock \showarticletitle{Certifying Geometric Robustness of Neural
  Networks}. In \bibinfo{booktitle}{\emph{Advances in Neural Information
  Processing Systems}}, \bibfield{editor}{\bibinfo{person}{H.~Wallach},
  \bibinfo{person}{H.~Larochelle}, \bibinfo{person}{A.~Beygelzimer},
  \bibinfo{person}{F.~d\textquotesingle Alch\'{e}-Buc},
  \bibinfo{person}{E.~Fox}, {and} \bibinfo{person}{R.~Garnett}} (Eds.),
  Vol.~\bibinfo{volume}{32}. \bibinfo{publisher}{Curran Associates, Inc.}
\newblock
\urldef\tempurl%
\url{https://proceedings.neurips.cc/paper_files/paper/2019/file/f7fa6aca028e7ff4ef62d75ed025fe76-Paper.pdf}
\showURL{%
\tempurl}


\bibitem[Bergstra and Bengio(2012)]%
        {RANDOM_SEARCH_HYPERPARAMETER}
\bibfield{author}{\bibinfo{person}{James Bergstra} {and}
  \bibinfo{person}{Yoshua Bengio}.} \bibinfo{year}{2012}\natexlab{}.
\newblock \showarticletitle{Random search for hyper-parameter optimization}.
\newblock \bibinfo{journal}{\emph{J. Mach. Learn. Res.}} \bibinfo{volume}{13},
  \bibinfo{number}{null} (\bibinfo{date}{Feb.} \bibinfo{year}{2012}),
  \bibinfo{pages}{281–305}.
\newblock
\showISSN{1532-4435}
\urldef\tempurl%
\url{https://doi.org/10.5555/2503308.2188395}
\showDOI{\tempurl}


\bibitem[Bojarski et~al\mbox{.}(2016)]%
        {AUTONOMOUS_DRIVING2}
\bibfield{author}{\bibinfo{person}{Mariusz Bojarski}, \bibinfo{person}{David~W.
  del Testa}, \bibinfo{person}{Daniel Dworakowski}, \bibinfo{person}{Bernhard
  Firner}, \bibinfo{person}{Beat Flepp}, \bibinfo{person}{Prasoon Goyal},
  \bibinfo{person}{Lawrence~D. Jackel}, \bibinfo{person}{Mathew Monfort},
  \bibinfo{person}{Urs Muller}, \bibinfo{person}{Jiakai Zhang},
  \bibinfo{person}{Xin Zhang}, \bibinfo{person}{Jake Zhao}, {and}
  \bibinfo{person}{Karol Zieba}.} \bibinfo{year}{2016}\natexlab{}.
\newblock \showarticletitle{End to End Learning for Self-Driving Cars}.
\newblock \bibinfo{journal}{\emph{ArXiv}}  \bibinfo{volume}{abs/1604.07316}
  (\bibinfo{year}{2016}).
\newblock
\urldef\tempurl%
\url{https://doi.org/10.48550/arXiv.1604.07316}
\showDOI{\tempurl}


\bibitem[Brázdil et~al\mbox{.}(2015)]%
        {MDP_VERIF}
\bibfield{author}{\bibinfo{person}{Tomáš Brázdil},
  \bibinfo{person}{Krishnendu Chatterjee}, \bibinfo{person}{Martin Chmelík},
  \bibinfo{person}{Vojtěch Forejt}, \bibinfo{person}{Jan Křetínský},
  \bibinfo{person}{Marta Kwiatkowska}, \bibinfo{person}{David Parker}, {and}
  \bibinfo{person}{Mateusz Ujma}.} \bibinfo{year}{2015}\natexlab{}.
\newblock \bibinfo{title}{Verification of Markov Decision Processes using
  Learning Algorithms}.
\newblock
\newblock
\urldef\tempurl%
\url{https://doi.org/10.48550/arXiv.1402.2967}
\showDOI{\tempurl}
\showeprint[arxiv]{1402.2967}~[cs.LO]


\bibitem[Cheng and Yan(2021)]%
        {CONTINUOUS}
\bibfield{author}{\bibinfo{person}{Chih{-}Hong Cheng} {and}
  \bibinfo{person}{Rongjie Yan}.} \bibinfo{year}{2021}\natexlab{}.
\newblock \showarticletitle{Continuous Safety Verification of Neural Networks}.
\newblock  (\bibinfo{year}{2021}), \bibinfo{pages}{1478--1483}.
\newblock
\urldef\tempurl%
\url{https://doi.org/10.23919/DATE51398.2021.9473994}
\showDOI{\tempurl}


\bibitem[Croce and Hein(2019)]%
        {ADVERSARIAL_ATTACKS3}
\bibfield{author}{\bibinfo{person}{Francesco Croce} {and}
  \bibinfo{person}{Matthias Hein}.} \bibinfo{year}{2019}\natexlab{}.
\newblock \showarticletitle{Sparse and Imperceivable Adversarial Attacks}. In
  \bibinfo{booktitle}{\emph{Proceedings of the IEEE/CVF International
  Conference on Computer Vision (ICCV)}}.
\newblock
\urldef\tempurl%
\url{https://doi.org/10.48550/arXiv.1909.05040}
\showDOI{\tempurl}


\bibitem[Deng(2012)]%
        {MNIST}
\bibfield{author}{\bibinfo{person}{Li Deng}.} \bibinfo{year}{2012}\natexlab{}.
\newblock \showarticletitle{The mnist database of handwritten digit images for
  machine learning research}.
\newblock \bibinfo{journal}{\emph{IEEE Signal Processing Magazine}}
  \bibinfo{volume}{29}, \bibinfo{number}{6} (\bibinfo{year}{2012}),
  \bibinfo{pages}{141--142}.
\newblock
\urldef\tempurl%
\url{https://doi.org/10.1109/MSP.2012.2211477}
\showDOI{\tempurl}


\bibitem[Dimitrakopoulos et~al\mbox{.}(2023)]%
        {MAB_HW_VERIF}
\bibfield{author}{\bibinfo{person}{G. Dimitrakopoulos}, \bibinfo{person}{E.
  Kallitsounakis}, \bibinfo{person}{Z. Takakis}, \bibinfo{person}{A.
  Stefanidis}, {and} \bibinfo{person}{C. Nicopoulos}.}
  \bibinfo{year}{2023}\natexlab{}.
\newblock \showarticletitle{Multi-Armed Bandits for Autonomous Test Application
  in RISC-V Processor Verification}. In \bibinfo{booktitle}{\emph{2023 12th
  International Conference on Modern Circuits and Systems Technologies
  (MOCAST)}}. \bibinfo{pages}{1--5}.
\newblock
\urldef\tempurl%
\url{https://doi.org/10.1109/MOCAST57943.2023.10176659}
\showDOI{\tempurl}


\bibitem[Esteva et~al\mbox{.}(2017)]%
        {SKIN_CANCER}
\bibfield{author}{\bibinfo{person}{Andre Esteva}, \bibinfo{person}{Brett
  Kuprel}, \bibinfo{person}{Roberto Novoa}, \bibinfo{person}{Justin Ko},
  \bibinfo{person}{Susan Swetter}, \bibinfo{person}{Helen Blau}, {and}
  \bibinfo{person}{Sebastian Thrun}.} \bibinfo{year}{2017}\natexlab{}.
\newblock \showarticletitle{Dermatologist-level classification of skin cancer
  with deep neural networks}.
\newblock \bibinfo{journal}{\emph{Nature}}  \bibinfo{volume}{542}
  (\bibinfo{date}{01} \bibinfo{year}{2017}).
\newblock
\urldef\tempurl%
\url{https://doi.org/10.1038/nature21056}
\showDOI{\tempurl}


\bibitem[{European Commission}(2020)]%
        {eucaiwhitepaper}
\bibfield{author}{\bibinfo{person}{{European Commission}}.}
  \bibinfo{year}{2020}\natexlab{}.
\newblock \bibinfo{booktitle}{\emph{{White Paper on Artificial Intelligence - A
  European approach to excellence and trust}}}.
\newblock {European Commission}.
\newblock
\urldef\tempurl%
\url{https://digital-strategy.ec.europa.eu/en/consultations/white-paper-artificial-intelligence-european-approach-excellence-and-trust}
\showURL{%
\tempurl}


\bibitem[Ferrari et~al\mbox{.}(2022)]%
        {MN_BAB}
\bibfield{author}{\bibinfo{person}{Claudio Ferrari},
  \bibinfo{person}{Mark~Niklas Muller}, \bibinfo{person}{Nikola Jovanovic},
  {and} \bibinfo{person}{Martin Vechev}.} \bibinfo{year}{2022}\natexlab{}.
\newblock \bibinfo{title}{Complete Verification via Multi-Neuron Relaxation
  Guided Branch-and-Bound}.
\newblock
\newblock
\urldef\tempurl%
\url{https://doi.org/10.48550/arXiv.2205.00263}
\showDOI{\tempurl}
\showeprint[arxiv]{2205.00263}~[cs.LG]


\bibitem[Fischer et~al\mbox{.}(2022)]%
        {SHARED_CERTIF}
\bibfield{author}{\bibinfo{person}{Marc Fischer}, \bibinfo{person}{Christian
  Sprecher}, \bibinfo{person}{Dimitar~I. Dimitrov}, \bibinfo{person}{Gagandeep
  Singh}, {and} \bibinfo{person}{Martin Vechev}.}
  \bibinfo{year}{2022}\natexlab{}.
\newblock \bibinfo{title}{Shared Certificates for Neural Network Verification}.
\newblock , \bibinfo{numpages}{127--148}~pages.
\newblock
\urldef\tempurl%
\url{https://doi.org/10.1007/978-3-031-13185-1\_7}
\showDOI{\tempurl}


\bibitem[Fischer et~al\mbox{.}(2023)]%
        {SHARED_CERTIF_ARXIV}
\bibfield{author}{\bibinfo{person}{Marc Fischer}, \bibinfo{person}{Christian
  Sprecher}, \bibinfo{person}{Dimitar~I. Dimitrov}, \bibinfo{person}{Gagandeep
  Singh}, {and} \bibinfo{person}{Martin Vechev}.}
  \bibinfo{year}{2023}\natexlab{}.
\newblock \bibinfo{title}{Shared Certificates for Neural Network Verification}.
\newblock
\newblock
\urldef\tempurl%
\url{https://doi.org/10.48550/arXiv.2109.00542}
\showDOI{\tempurl}
\showeprint[arxiv]{2109.00542}~[cs.LG]


\bibitem[Gehr et~al\mbox{.}(2018)]%
        {AI2}
\bibfield{author}{\bibinfo{person}{Timon Gehr}, \bibinfo{person}{Matthew
  Mirman}, \bibinfo{person}{Dana Drachsler-Cohen}, \bibinfo{person}{Petar
  Tsankov}, \bibinfo{person}{Swarat Chaudhuri}, {and} \bibinfo{person}{Martin
  Vechev}.} \bibinfo{year}{2018}\natexlab{}.
\newblock \showarticletitle{AI2: Safety and Robustness Certification of Neural
  Networks with Abstract Interpretation}. In \bibinfo{booktitle}{\emph{2018
  IEEE Symposium on Security and Privacy (SP)}}. \bibinfo{pages}{3--18}.
\newblock
\urldef\tempurl%
\url{https://doi.org/10.1109/SP.2018.00058}
\showDOI{\tempurl}


\bibitem[Goodfellow et~al\mbox{.}(2015)]%
        {EXPLAINING_ADVERSARIAL_EXAMPLES}
\bibfield{author}{\bibinfo{person}{Ian~J. Goodfellow},
  \bibinfo{person}{Jonathon Shlens}, {and} \bibinfo{person}{Christian
  Szegedy}.} \bibinfo{year}{2015}\natexlab{}.
\newblock \showarticletitle{Explaining and Harnessing Adversarial Examples}. In
  \bibinfo{booktitle}{\emph{3rd International Conference on Learning
  Representations, {ICLR} 2015, San Diego, CA, USA, May 7-9, 2015, Conference
  Track Proceedings}}, \bibfield{editor}{\bibinfo{person}{Yoshua Bengio} {and}
  \bibinfo{person}{Yann LeCun}} (Eds.).
\newblock
\urldef\tempurl%
\url{https://doi.org/10.48550/arXiv.1412.6572}
\showDOI{\tempurl}


\bibitem[Graves and Jaitly(2014)]%
        {SPEECH_RECOGNITION}
\bibfield{author}{\bibinfo{person}{Alex Graves} {and} \bibinfo{person}{Navdeep
  Jaitly}.} \bibinfo{year}{2014}\natexlab{}.
\newblock \showarticletitle{Towards End-To-End Speech Recognition with
  Recurrent Neural Networks}. In \bibinfo{booktitle}{\emph{Proceedings of the
  31st International Conference on Machine Learning}}
  \emph{(\bibinfo{series}{Proceedings of Machine Learning Research},
  Vol.~\bibinfo{volume}{32})}, \bibfield{editor}{\bibinfo{person}{Eric~P. Xing}
  {and} \bibinfo{person}{Tony Jebara}} (Eds.). \bibinfo{publisher}{PMLR},
  \bibinfo{address}{Bejing, China}, \bibinfo{pages}{1764--1772}.
\newblock
\urldef\tempurl%
\url{https://doi.org/10.48550/arXiv.1701.02720}
\showDOI{\tempurl}


\bibitem[Guo et~al\mbox{.}(2019)]%
        {ADVERSARIAL_ATTACKS4}
\bibfield{author}{\bibinfo{person}{Chuan Guo}, \bibinfo{person}{Jacob Gardner},
  \bibinfo{person}{Yurong You}, \bibinfo{person}{Andrew~Gordon Wilson}, {and}
  \bibinfo{person}{Kilian Weinberger}.} \bibinfo{year}{2019}\natexlab{}.
\newblock \showarticletitle{Simple Black-box Adversarial Attacks}. In
  \bibinfo{booktitle}{\emph{Proceedings of the 36th International Conference on
  Machine Learning}} \emph{(\bibinfo{series}{Proceedings of Machine Learning
  Research}, Vol.~\bibinfo{volume}{97})},
  \bibfield{editor}{\bibinfo{person}{Kamalika Chaudhuri} {and}
  \bibinfo{person}{Ruslan Salakhutdinov}} (Eds.). \bibinfo{publisher}{PMLR},
  \bibinfo{pages}{2484--2493}.
\newblock
\urldef\tempurl%
\url{https://doi.org/10.48550/arXiv.1905.07121}
\showDOI{\tempurl}


\bibitem[{Gurobi Optimization, LLC}(2023)]%
        {GUROBI}
\bibfield{author}{\bibinfo{person}{{Gurobi Optimization, LLC}}.}
  \bibinfo{year}{2023}\natexlab{}.
\newblock \bibinfo{title}{{Gurobi Optimizer Reference Manual}}.
\newblock
\newblock
\urldef\tempurl%
\url{https://www.gurobi.com}
\showURL{%
\tempurl}


\bibitem[Huang et~al\mbox{.}(2021)]%
        {L2_PERTURBATION}
\bibfield{author}{\bibinfo{person}{Yujia Huang}, \bibinfo{person}{Huan Zhang},
  \bibinfo{person}{Yuanyuan Shi}, \bibinfo{person}{J~Zico Kolter}, {and}
  \bibinfo{person}{Anima Anandkumar}.} \bibinfo{year}{2021}\natexlab{}.
\newblock \bibinfo{title}{Training Certifiably Robust Neural Networks with
  Efficient Local Lipschitz Bounds}.
\newblock
\newblock
\urldef\tempurl%
\url{https://doi.org/10.48550/arXiv.2111.01395}
\showDOI{\tempurl}
\showeprint[arxiv]{2111.01395}~[cs.LG]


\bibitem[Ilyas et~al\mbox{.}(2018)]%
        {ADVERSARIAL_ATTACKS1}
\bibfield{author}{\bibinfo{person}{Andrew Ilyas}, \bibinfo{person}{Logan
  Engstrom}, \bibinfo{person}{Anish Athalye}, {and} \bibinfo{person}{Jessy
  Lin}.} \bibinfo{year}{2018}\natexlab{}.
\newblock \showarticletitle{Black-box Adversarial Attacks with Limited Queries
  and Information}. In \bibinfo{booktitle}{\emph{Proceedings of the 35th
  International Conference on Machine Learning}}
  \emph{(\bibinfo{series}{Proceedings of Machine Learning Research},
  Vol.~\bibinfo{volume}{80})}, \bibfield{editor}{\bibinfo{person}{Jennifer Dy}
  {and} \bibinfo{person}{Andreas Krause}} (Eds.). \bibinfo{publisher}{PMLR},
  \bibinfo{pages}{2137--2146}.
\newblock
\urldef\tempurl%
\url{https://doi.org/10.48550/arXiv.1804.08598}
\showDOI{\tempurl}


\bibitem[Javed and Shah(2002)]%
        {SURVEILLANCE1}
\bibfield{author}{\bibinfo{person}{Omar Javed} {and} \bibinfo{person}{Mubarak
  Shah}.} \bibinfo{year}{2002}\natexlab{}.
\newblock \showarticletitle{Tracking and Object Classification for Automated
  Surveillance}. In \bibinfo{booktitle}{\emph{Computer Vision --- ECCV 2002}},
  \bibfield{editor}{\bibinfo{person}{Anders Heyden}, \bibinfo{person}{Gunnar
  Sparr}, \bibinfo{person}{Mads Nielsen}, {and} \bibinfo{person}{Peter
  Johansen}} (Eds.). \bibinfo{publisher}{Springer Berlin Heidelberg},
  \bibinfo{address}{Berlin, Heidelberg}, \bibinfo{pages}{343--357}.
\newblock
\showISBNx{978-3-540-47979-6}
\urldef\tempurl%
\url{https://doi.org/10.1007/3-540-47979-1_23}
\showDOI{\tempurl}


\bibitem[Kabaha and Drachsler-Cohen(2022)]%
        {VEEP}
\bibfield{author}{\bibinfo{person}{Anan Kabaha} {and} \bibinfo{person}{Dana
  Drachsler-Cohen}.} \bibinfo{year}{2022}\natexlab{}.
\newblock \showarticletitle{Boosting Robustness Verification of Semantic
  Feature Neighborhoods}. In \bibinfo{booktitle}{\emph{Static Analysis: 29th
  International Symposium {SAS}}} (Auckland, New Zealand).
  \bibinfo{publisher}{Springer-Verlag}, \bibinfo{address}{Berlin, Heidelberg},
  \bibinfo{pages}{299–324}.
\newblock
\showISBNx{978-3-031-22307-5}
\urldef\tempurl%
\url{https://doi.org/10.1007/978-3-031-22308-2_14}
\showDOI{\tempurl}


\bibitem[Kabaha and Drachsler-Cohen(2024)]%
        {VHAGAR}
\bibfield{author}{\bibinfo{person}{Anan Kabaha} {and} \bibinfo{person}{Dana
  Drachsler-Cohen}.} \bibinfo{year}{2024}\natexlab{}.
\newblock \showarticletitle{Verification of Neural Networks’ Global
  Robustness}.
\newblock \bibinfo{journal}{\emph{Proc. ACM Program. Lang.}}
  \bibinfo{volume}{8}, \bibinfo{number}{OOPSLA1}, Article
  \bibinfo{articleno}{130} (\bibinfo{date}{April} \bibinfo{year}{2024}),
  \bibinfo{numpages}{30}~pages.
\newblock
\urldef\tempurl%
\url{https://doi.org/10.1145/3649847}
\showDOI{\tempurl}


\bibitem[Karim et~al\mbox{.}(2021)]%
        {ADVERSARIAL_ATTACKS2}
\bibfield{author}{\bibinfo{person}{Fazle Karim}, \bibinfo{person}{Somshubra
  Majumdar}, {and} \bibinfo{person}{Houshang Darabi}.}
  \bibinfo{year}{2021}\natexlab{}.
\newblock \showarticletitle{Adversarial Attacks on Time Series}.
\newblock \bibinfo{journal}{\emph{IEEE Transactions on Pattern Analysis and
  Machine Intelligence}} \bibinfo{volume}{43}, \bibinfo{number}{10}
  (\bibinfo{year}{2021}), \bibinfo{pages}{3309--3320}.
\newblock
\urldef\tempurl%
\url{https://doi.org/10.1109/TPAMI.2020.2986319}
\showDOI{\tempurl}


\bibitem[Katz et~al\mbox{.}(2017)]%
        {RELUPLEX}
\bibfield{author}{\bibinfo{person}{Guy Katz}, \bibinfo{person}{Clark Barrett},
  \bibinfo{person}{David~L. Dill}, \bibinfo{person}{Kyle Julian}, {and}
  \bibinfo{person}{Mykel~J. Kochenderfer}.} \bibinfo{year}{2017}\natexlab{}.
\newblock \showarticletitle{Reluplex: An Efficient SMT Solver for Verifying
  Deep Neural Networks}. In \bibinfo{booktitle}{\emph{Computer Aided
  Verification}}, \bibfield{editor}{\bibinfo{person}{Rupak Majumdar} {and}
  \bibinfo{person}{Viktor Kun{\v{c}}ak}} (Eds.). \bibinfo{publisher}{Springer
  International Publishing}, \bibinfo{address}{Cham}, \bibinfo{pages}{97--117}.
\newblock
\showISBNx{978-3-319-63387-9}
\urldef\tempurl%
\url{https://doi.org/10.48550/arXiv.1702.01135}
\showDOI{\tempurl}


\bibitem[Katz et~al\mbox{.}(2019)]%
        {MARABOU}
\bibfield{author}{\bibinfo{person}{Guy Katz}, \bibinfo{person}{Derek~A. Huang},
  \bibinfo{person}{Duligur Ibeling}, \bibinfo{person}{Kyle Julian},
  \bibinfo{person}{Christopher Lazarus}, \bibinfo{person}{Rachel Lim},
  \bibinfo{person}{Parth Shah}, \bibinfo{person}{Shantanu Thakoor},
  \bibinfo{person}{Haoze Wu}, \bibinfo{person}{Aleksandar Zelji{\'{c}}},
  \bibinfo{person}{David~L. Dill}, \bibinfo{person}{Mykel~J. Kochenderfer},
  {and} \bibinfo{person}{Clark Barrett}.} \bibinfo{year}{2019}\natexlab{}.
\newblock \showarticletitle{The Marabou Framework for Verification and Analysis
  of Deep Neural Networks}. In \bibinfo{booktitle}{\emph{Computer Aided
  Verification}}, \bibfield{editor}{\bibinfo{person}{Isil Dillig} {and}
  \bibinfo{person}{Serdar Tasiran}} (Eds.). \bibinfo{publisher}{Springer
  International Publishing}, \bibinfo{address}{Cham},
  \bibinfo{pages}{443--452}.
\newblock
\showISBNx{978-3-030-25540-4}
\urldef\tempurl%
\url{https://doi.org/10.1007/978-3-030-25540-4\_26}
\showDOI{\tempurl}


\bibitem[Kingma and Ba(2014)]%
        {ADAM}
\bibfield{author}{\bibinfo{person}{Diederik Kingma} {and}
  \bibinfo{person}{Jimmy Ba}.} \bibinfo{year}{2014}\natexlab{}.
\newblock \showarticletitle{Adam: A Method for Stochastic Optimization}.
\newblock \bibinfo{journal}{\emph{International Conference on Learning
  Representations}} (\bibinfo{date}{12} \bibinfo{year}{2014}).
\newblock
\urldef\tempurl%
\url{https://doi.org/10.48550/arXiv.1412.6980}
\showDOI{\tempurl}


\bibitem[Krizhevsky(2012)]%
        {CIFAR}
\bibfield{author}{\bibinfo{person}{Alex Krizhevsky}.}
  \bibinfo{year}{2012}\natexlab{}.
\newblock \showarticletitle{Learning Multiple Layers of Features from Tiny
  Images}.
\newblock \bibinfo{journal}{\emph{University of Toronto}} (\bibinfo{date}{05}
  \bibinfo{year}{2012}).
\newblock


\bibitem[Krizhevsky et~al\mbox{.}(2012)]%
        {ALEXNET}
\bibfield{author}{\bibinfo{person}{Alex Krizhevsky}, \bibinfo{person}{Ilya
  Sutskever}, {and} \bibinfo{person}{Geoffrey~E Hinton}.}
  \bibinfo{year}{2012}\natexlab{}.
\newblock \showarticletitle{ImageNet Classification with Deep Convolutional
  Neural Networks}. In \bibinfo{booktitle}{\emph{Advances in Neural Information
  Processing Systems}}, \bibfield{editor}{\bibinfo{person}{F.~Pereira},
  \bibinfo{person}{C.J. Burges}, \bibinfo{person}{L.~Bottou}, {and}
  \bibinfo{person}{K.Q. Weinberger}} (Eds.), Vol.~\bibinfo{volume}{25}.
  \bibinfo{publisher}{Curran Associates, Inc.}
\newblock
\urldef\tempurl%
\url{https://doi.org/10.1145/3065386}
\showDOI{\tempurl}


\bibitem[Leino et~al\mbox{.}(2021)]%
        {GLOBAL_ROBUST}
\bibfield{author}{\bibinfo{person}{Klas Leino}, \bibinfo{person}{Zifan Wang},
  {and} \bibinfo{person}{Matt Fredrikson}.} \bibinfo{year}{2021}\natexlab{}.
\newblock \bibinfo{title}{Globally-Robust Neural Networks}.
\newblock
\newblock
\urldef\tempurl%
\url{https://doi.org/10.48550/arXiv.2102.08452}
\showDOI{\tempurl}
\showeprint[arxiv]{2102.08452}~[cs.LG]


\bibitem[Lopez et~al\mbox{.}(2023)]%
        {NNV}
\bibfield{author}{\bibinfo{person}{Diego~Manzanas Lopez},
  \bibinfo{person}{Sung~Woo Choi}, \bibinfo{person}{Hoang-Dung Tran}, {and}
  \bibinfo{person}{Taylor~T. Johnson}.} \bibinfo{year}{2023}\natexlab{}.
\newblock \showarticletitle{NNV 2.0: The Neural Network Verification Tool}. In
  \bibinfo{booktitle}{\emph{Computer Aided Verification}},
  \bibfield{editor}{\bibinfo{person}{Constantin Enea} {and}
  \bibinfo{person}{Akash Lal}} (Eds.). \bibinfo{publisher}{Springer Nature
  Switzerland}, \bibinfo{address}{Cham}, \bibinfo{pages}{397--412}.
\newblock
\showISBNx{978-3-031-37703-7}
\urldef\tempurl%
\url{https://doi.org/10.1007/978-3-031-37703-7_19}
\showDOI{\tempurl}


\bibitem[MacQueen(1967)]%
        {KMEANS}
\bibfield{author}{\bibinfo{person}{J MacQueen}.}
  \bibinfo{year}{1967}\natexlab{}.
\newblock \showarticletitle{Some methods for classification and analysis of
  multivariate observations}. In \bibinfo{booktitle}{\emph{Proceedings of 5-th
  Berkeley Symposium on Mathematical Statistics and Probability/University of
  California Press}}.
\newblock


\bibitem[Madry et~al\mbox{.}(2018)]%
        {PGD}
\bibfield{author}{\bibinfo{person}{Aleksander Madry},
  \bibinfo{person}{Aleksandar Makelov}, \bibinfo{person}{Ludwig Schmidt},
  \bibinfo{person}{Dimitris Tsipras}, {and} \bibinfo{person}{Adrian Vladu}.}
  \bibinfo{year}{2018}\natexlab{}.
\newblock \showarticletitle{Towards Deep Learning Models Resistant to
  Adversarial Attacks}. In \bibinfo{booktitle}{\emph{6th International
  Conference on Learning Representations, {ICLR} 2018, Vancouver, BC, Canada,
  April 30 - May 3, 2018, Conference Track Proceedings}}.
  \bibinfo{publisher}{OpenReview.net}.
\newblock
\urldef\tempurl%
\url{https://doi.org/10.48550/arXiv.1706.06083}
\showDOI{\tempurl}


\bibitem[Mirman et~al\mbox{.}(2018)]%
        {CONV_ARCHITECTURES}
\bibfield{author}{\bibinfo{person}{Matthew Mirman}, \bibinfo{person}{Timon
  Gehr}, {and} \bibinfo{person}{Martin Vechev}.}
  \bibinfo{year}{2018}\natexlab{}.
\newblock \showarticletitle{Differentiable Abstract Interpretation for Provably
  Robust Neural Networks}. In \bibinfo{booktitle}{\emph{International
  Conference on Machine Learning}}. \bibinfo{pages}{3575--3583}.
\newblock
\urldef\tempurl%
\url{https://files.sri.inf.ethz.ch/website/papers/icml18-diffai.pdf}
\showURL{%
\tempurl}


\bibitem[Mohapatra et~al\mbox{.}(2020)]%
        {SEMANTIFY_NN}
\bibfield{author}{\bibinfo{person}{Jeet Mohapatra}, \bibinfo{person}{Tsui-Wei
  Weng}, \bibinfo{person}{Pin-Yu Chen}, \bibinfo{person}{Sijia Liu}, {and}
  \bibinfo{person}{Luca Daniel}.} \bibinfo{year}{2020}\natexlab{}.
\newblock \showarticletitle{Towards Verifying Robustness of Neural Networks
  Against A Family of Semantic Perturbations}. In
  \bibinfo{booktitle}{\emph{Proceedings of the IEEE/CVF Conference on Computer
  Vision and Pattern Recognition (CVPR)}}.
\newblock
\urldef\tempurl%
\url{https://doi.org/10.48550/arXiv.1912.09533}
\showDOI{\tempurl}


\bibitem[Müller et~al\mbox{.}(2021)]%
        {GPUPOLY}
\bibfield{author}{\bibinfo{person}{Christoph Müller},
  \bibinfo{person}{François Serre}, \bibinfo{person}{Gagandeep Singh},
  \bibinfo{person}{Markus Püschel}, {and} \bibinfo{person}{Martin Vechev}.}
  \bibinfo{year}{2021}\natexlab{}.
\newblock \bibinfo{title}{Scaling Polyhedral Neural Network Verification on
  GPUs}.
\newblock
\newblock
\urldef\tempurl%
\url{https://doi.org/10.48550/arXiv.2007.10868}
\showDOI{\tempurl}
\showeprint[arxiv]{2007.10868}~[cs.LG]


\bibitem[Ostrovsky et~al\mbox{.}(2022)]%
        {CNN_ABSTRACT_REFINE}
\bibfield{author}{\bibinfo{person}{Matan Ostrovsky}, \bibinfo{person}{Clark
  Barrett}, {and} \bibinfo{person}{Guy Katz}.} \bibinfo{year}{2022}\natexlab{}.
\newblock \showarticletitle{An Abstraction-Refinement Approach to Verifying
  Convolutional Neural Networks}. In \bibinfo{booktitle}{\emph{Automated
  Technology for Verification and Analysis}},
  \bibfield{editor}{\bibinfo{person}{Ahmed Bouajjani},
  \bibinfo{person}{Luk{\'a}{\v{s}} Hol{\'i}k}, {and} \bibinfo{person}{Zhilin
  Wu}} (Eds.). \bibinfo{publisher}{Springer International Publishing},
  \bibinfo{address}{Cham}, \bibinfo{pages}{391--396}.
\newblock
\showISBNx{978-3-031-19992-9}
\urldef\tempurl%
\url{https://doi.org/10.1007/978-3-031-19992-9\_25}
\showDOI{\tempurl}


\bibitem[Redmon et~al\mbox{.}(2016)]%
        {YOLO}
\bibfield{author}{\bibinfo{person}{Joseph Redmon},
  \bibinfo{person}{Santosh~Kumar Divvala}, \bibinfo{person}{Ross~B. Girshick},
  {and} \bibinfo{person}{Ali Farhadi}.} \bibinfo{year}{2016}\natexlab{}.
\newblock \showarticletitle{You Only Look Once: Unified, Real-Time Object
  Detection}. In \bibinfo{booktitle}{\emph{2016 {IEEE} Conference on Computer
  Vision and Pattern Recognition, {CVPR} 2016, Las Vegas, NV, USA, June 27-30,
  2016}}. \bibinfo{publisher}{{IEEE} Computer Society},
  \bibinfo{pages}{779--788}.
\newblock
\urldef\tempurl%
\url{https://doi.org/10.1109/CVPR.2016.91}
\showDOI{\tempurl}


\bibitem[Shapira et~al\mbox{.}(2023)]%
        {CALZONE}
\bibfield{author}{\bibinfo{person}{Yuval Shapira}, \bibinfo{person}{Eran
  Avneri}, {and} \bibinfo{person}{Dana Drachsler-Cohen}.}
  \bibinfo{year}{2023}\natexlab{}.
\newblock \showarticletitle{Deep Learning Robustness Verification for Few-Pixel
  Attacks}.
\newblock \bibinfo{journal}{\emph{Proc. ACM Program. Lang.}}
  \bibinfo{volume}{7}, \bibinfo{number}{OOPSLA1}, Article
  \bibinfo{articleno}{90} (\bibinfo{date}{apr} \bibinfo{year}{2023}),
  \bibinfo{numpages}{28}~pages.
\newblock
\urldef\tempurl%
\url{https://doi.org/10.1145/3586042}
\showDOI{\tempurl}


\bibitem[Shapira et~al\mbox{.}(2024)]%
        {COVERD}
\bibfield{author}{\bibinfo{person}{Yuval Shapira}, \bibinfo{person}{Naor
  Wiesel}, \bibinfo{person}{Shahar Shabelman}, {and} \bibinfo{person}{Dana
  Drachsler-Cohen}.} \bibinfo{year}{2024}\natexlab{}.
\newblock \showarticletitle{Boosting Few-Pixel Robustness Verification via
  Covering Verification Designs}. In \bibinfo{booktitle}{\emph{Computer Aided
  Verification: 36th International Conference, CAV 2024, Montreal, QC, Canada,
  July 24–27, 2024, Proceedings, Part II}} (Montreal, QC, Canada).
  \bibinfo{publisher}{Springer-Verlag}, \bibinfo{address}{Berlin, Heidelberg},
  \bibinfo{pages}{377–400}.
\newblock
\showISBNx{978-3-031-65629-3}
\urldef\tempurl%
\url{https://doi.org/10.1007/978-3-031-65630-9_19}
\showDOI{\tempurl}


\bibitem[Singh et~al\mbox{.}(2018)]%
        {DEEPZ}
\bibfield{author}{\bibinfo{person}{Gagandeep Singh}, \bibinfo{person}{Timon
  Gehr}, \bibinfo{person}{Matthew Mirman}, \bibinfo{person}{Markus
  P\"{u}schel}, {and} \bibinfo{person}{Martin Vechev}.}
  \bibinfo{year}{2018}\natexlab{}.
\newblock \showarticletitle{Fast and Effective Robustness Certification}. In
  \bibinfo{booktitle}{\emph{Advances in Neural Information Processing
  Systems}}, \bibfield{editor}{\bibinfo{person}{S.~Bengio},
  \bibinfo{person}{H.~Wallach}, \bibinfo{person}{H.~Larochelle},
  \bibinfo{person}{K.~Grauman}, \bibinfo{person}{N.~Cesa-Bianchi}, {and}
  \bibinfo{person}{R.~Garnett}} (Eds.), Vol.~\bibinfo{volume}{31}.
  \bibinfo{publisher}{Curran Associates, Inc.}
\newblock
\urldef\tempurl%
\url{https://doi.org/10.5555/3327546.3327739}
\showDOI{\tempurl}


\bibitem[Singh et~al\mbox{.}(2019a)]%
        {DEEPPOLY}
\bibfield{author}{\bibinfo{person}{Gagandeep Singh}, \bibinfo{person}{Timon
  Gehr}, \bibinfo{person}{Markus P\"{u}schel}, {and} \bibinfo{person}{Martin
  Vechev}.} \bibinfo{year}{2019}\natexlab{a}.
\newblock \showarticletitle{An abstract domain for certifying neural networks}.
\newblock \bibinfo{journal}{\emph{Proc. ACM Program. Lang.}}
  \bibinfo{volume}{3}, \bibinfo{number}{POPL}, Article \bibinfo{articleno}{41}
  (\bibinfo{date}{jan} \bibinfo{year}{2019}), \bibinfo{numpages}{30}~pages.
\newblock
\urldef\tempurl%
\url{https://doi.org/10.1145/3290354}
\showDOI{\tempurl}


\bibitem[Singh et~al\mbox{.}(2019b)]%
        {REFINE_ZONO}
\bibfield{author}{\bibinfo{person}{Gagandeep Singh}, \bibinfo{person}{Timon
  Gehr}, \bibinfo{person}{Markus Püschel}, {and} \bibinfo{person}{Martin
  Vechev}.} \bibinfo{year}{2019}\natexlab{b}.
\newblock \showarticletitle{Robustness Certification with Refinement}. In
  \bibinfo{booktitle}{\emph{International Conference on Learning
  Representations}}.
\newblock
\urldef\tempurl%
\url{https://openreview.net/forum?id=HJgeEh09KQ}
\showURL{%
\tempurl}


\bibitem[Sutskever et~al\mbox{.}(2014)]%
        {SEQ_TO_SEQ}
\bibfield{author}{\bibinfo{person}{Ilya Sutskever}, \bibinfo{person}{Oriol
  Vinyals}, {and} \bibinfo{person}{Quoc~V. Le}.}
  \bibinfo{year}{2014}\natexlab{}.
\newblock \showarticletitle{Sequence to sequence learning with neural
  networks}. In \bibinfo{booktitle}{\emph{Proceedings of the 28th International
  Conference on Neural Information Processing Systems - Volume 2}} (Montreal,
  Canada) \emph{(\bibinfo{series}{NIPS'14})}. \bibinfo{publisher}{MIT Press},
  \bibinfo{address}{Cambridge, MA, USA}, \bibinfo{pages}{3104–3112}.
\newblock
\urldef\tempurl%
\url{https://doi.org/10.48550/arXiv.1409.3215}
\showDOI{\tempurl}


\bibitem[Szegedy et~al\mbox{.}(2013)]%
        {INTRIGUING_PROP}
\bibfield{author}{\bibinfo{person}{Christian Szegedy},
  \bibinfo{person}{Wojciech Zaremba}, \bibinfo{person}{Ilya Sutskever},
  \bibinfo{person}{Joan Bruna}, \bibinfo{person}{Dumitru Erhan},
  \bibinfo{person}{Ian Goodfellow}, {and} \bibinfo{person}{Rob Fergus}.}
  \bibinfo{year}{2013}\natexlab{}.
\newblock \showarticletitle{Intriguing properties of neural networks}.
\newblock  (\bibinfo{date}{12} \bibinfo{year}{2013}).
\newblock
\urldef\tempurl%
\url{https://doi.org/10.48550/arXiv.1312.6199}
\showDOI{\tempurl}


\bibitem[Thompson(1933)]%
        {THOMPSON_SAMPLING}
\bibfield{author}{\bibinfo{person}{William~R. Thompson}.}
  \bibinfo{year}{1933}\natexlab{}.
\newblock \showarticletitle{On the Likelihood that One Unknown Probability
  Exceeds Another in View of the Evidence of Two Samples}.
\newblock \bibinfo{journal}{\emph{Biometrika}} \bibinfo{volume}{25},
  \bibinfo{number}{3/4} (\bibinfo{year}{1933}), \bibinfo{pages}{285--294}.
\newblock
\showISSN{00063444}
\urldef\tempurl%
\url{https://doi.org/10.2307/2332286}
\showDOI{\tempurl}


\bibitem[Tishby and Zaslavsky(2015)]%
        {BOTTLENECK}
\bibfield{author}{\bibinfo{person}{Naftali Tishby} {and} \bibinfo{person}{Noga
  Zaslavsky}.} \bibinfo{year}{2015}\natexlab{}.
\newblock \showarticletitle{Deep Learning and the Information Bottleneck
  Principle}.
\newblock \bibinfo{journal}{\emph{2015 IEEE Information Theory Workshop, ITW
  2015}} (\bibinfo{date}{03} \bibinfo{year}{2015}).
\newblock
\urldef\tempurl%
\url{https://doi.org/10.1109/ITW.2015.7133169}
\showDOI{\tempurl}


\bibitem[Tjeng et~al\mbox{.}(2019)]%
        {MIPVERIFY}
\bibfield{author}{\bibinfo{person}{Vincent Tjeng}, \bibinfo{person}{Kai Xiao},
  {and} \bibinfo{person}{Russ Tedrake}.} \bibinfo{year}{2019}\natexlab{}.
\newblock \showarticletitle{Evaluating Robustness of Neural Networks with Mixed
  Integer Programming}.
\newblock \bibinfo{journal}{\emph{ICLR}} (\bibinfo{year}{2019}).
\newblock
\urldef\tempurl%
\url{https://doi.org/10.48550/arXiv.1711.07356}
\showDOI{\tempurl}


\bibitem[Tran et~al\mbox{.}(2020)]%
        {IMAGESTARS}
\bibfield{author}{\bibinfo{person}{Hoang-Dung Tran}, \bibinfo{person}{Stanley
  Bak}, \bibinfo{person}{Weiming Xiang}, {and} \bibinfo{person}{Taylor~T.
  Johnson}.} \bibinfo{year}{2020}\natexlab{}.
\newblock \showarticletitle{Verification of Deep Convolutional Neural Networks
  Using ImageStars}. In \bibinfo{booktitle}{\emph{Computer Aided
  Verification}}, \bibfield{editor}{\bibinfo{person}{Shuvendu~K. Lahiri} {and}
  \bibinfo{person}{Chao Wang}} (Eds.). \bibinfo{publisher}{Springer
  International Publishing}, \bibinfo{address}{Cham}, \bibinfo{pages}{18--42}.
\newblock
\showISBNx{978-3-030-53288-8}
\urldef\tempurl%
\url{https://doi.org/10.1007/978-3-030-53288-8\_2}
\showDOI{\tempurl}


\bibitem[Tran et~al\mbox{.}(2019)]%
        {STAR_DOMAIN2}
\bibfield{author}{\bibinfo{person}{Hoang-Dung Tran}, \bibinfo{person}{Diago
  Manzanas~Lopez}, \bibinfo{person}{Patrick Musau}, \bibinfo{person}{Xiaodong
  Yang}, \bibinfo{person}{Luan~Viet Nguyen}, \bibinfo{person}{Weiming Xiang},
  {and} \bibinfo{person}{Taylor~T. Johnson}.} \bibinfo{year}{2019}\natexlab{}.
\newblock \showarticletitle{Star-Based Reachability Analysis of Deep Neural
  Networks}. In \bibinfo{booktitle}{\emph{Formal Methods -- The Next 30
  Years}}, \bibfield{editor}{\bibinfo{person}{Maurice~H. ter Beek},
  \bibinfo{person}{Annabelle McIver}, {and} \bibinfo{person}{Jos{\'e}~N.
  Oliveira}} (Eds.). \bibinfo{publisher}{Springer International Publishing},
  \bibinfo{address}{Cham}, \bibinfo{pages}{670--686}.
\newblock
\showISBNx{978-3-030-30942-8}
\urldef\tempurl%
\url{https://doi.org/10.1007/978-3-030-30942-8\_39}
\showDOI{\tempurl}


\bibitem[Tzour-Shaday and Drachsler-Cohen(2025)]%
        {BaVerLy-artifact}
\bibfield{author}{\bibinfo{person}{Saar Tzour-Shaday} {and}
  \bibinfo{person}{Dana Drachsler-Cohen}.} \bibinfo{year}{2025}\natexlab{}.
\newblock \bibinfo{title}{Reproduction Package for Article `Mini-Batch
  Robustness Verification of Deep Neural Networks'}.
\newblock \bibinfo{howpublished}{Zenodo}.
\newblock
\urldef\tempurl%
\url{https://doi.org/10.5281/zenodo.16892960}
\showDOI{\tempurl}


\bibitem[Ugare et~al\mbox{.}(2023)]%
        {IVAN}
\bibfield{author}{\bibinfo{person}{Shubham Ugare}, \bibinfo{person}{Debangshu
  Banerjee}, \bibinfo{person}{Sasa Misailovic}, {and}
  \bibinfo{person}{Gagandeep Singh}.} \bibinfo{year}{2023}\natexlab{}.
\newblock \showarticletitle{Incremental Verification of Neural Networks}.
\newblock \bibinfo{journal}{\emph{Proceedings of the ACM on Programming
  Languages}} \bibinfo{volume}{7}, \bibinfo{number}{PLDI} (\bibinfo{date}{June}
  \bibinfo{year}{2023}), \bibinfo{pages}{1920–1945}.
\newblock
\showISSN{2475-1421}
\urldef\tempurl%
\url{https://doi.org/10.1145/3591299}
\showDOI{\tempurl}


\bibitem[Ugare et~al\mbox{.}(2022)]%
        {FANC}
\bibfield{author}{\bibinfo{person}{Shubham Ugare}, \bibinfo{person}{Gagandeep
  Singh}, {and} \bibinfo{person}{Sasa Misailovic}.}
  \bibinfo{year}{2022}\natexlab{}.
\newblock \showarticletitle{Proof transfer for fast certification of multiple
  approximate neural networks}.
\newblock \bibinfo{journal}{\emph{Proc. ACM Program. Lang.}}
  \bibinfo{volume}{6}, \bibinfo{number}{OOPSLA1}, Article
  \bibinfo{articleno}{75} (\bibinfo{date}{April} \bibinfo{year}{2022}),
  \bibinfo{numpages}{29}~pages.
\newblock
\urldef\tempurl%
\url{https://doi.org/10.1145/3527319}
\showDOI{\tempurl}


\bibitem[Wan et~al\mbox{.}(2020)]%
        {TARGET_LABELS}
\bibfield{author}{\bibinfo{person}{Wenjie Wan}, \bibinfo{person}{Zhaodi Zhang},
  \bibinfo{person}{Yiwei Zhu}, \bibinfo{person}{Min Zhang}, {and}
  \bibinfo{person}{Fu Song}.} \bibinfo{year}{2020}\natexlab{}.
\newblock \showarticletitle{Accelerating Robustness Verification of Deep Neural
  Networks Guided by Target Labels}.
\newblock \bibinfo{journal}{\emph{ArXiv}}  \bibinfo{volume}{abs/2007.08520}
  (\bibinfo{year}{2020}).
\newblock
\urldef\tempurl%
\url{https://doi.org/10.48550/arXiv.2007.08520}
\showDOI{\tempurl}


\bibitem[Wang et~al\mbox{.}(2023)]%
        {GEOMETRIC_PERTURBATIONS2}
\bibfield{author}{\bibinfo{person}{Fu Wang}, \bibinfo{person}{Peipei Xu},
  \bibinfo{person}{Wenjie Ruan}, {and} \bibinfo{person}{Xiaowei Huang}.}
  \bibinfo{year}{2023}\natexlab{}.
\newblock \showarticletitle{Towards verifying the geometric robustness of
  large-scale neural networks}. In \bibinfo{booktitle}{\emph{Proceedings of the
  Thirty-Seventh AAAI Conference on Artificial Intelligence and Thirty-Fifth
  Conference on Innovative Applications of Artificial Intelligence and
  Thirteenth Symposium on Educational Advances in Artificial Intelligence}}
  \emph{(\bibinfo{series}{AAAI'23/IAAI'23/EAAI'23})}. \bibinfo{publisher}{AAAI
  Press}, Article \bibinfo{articleno}{1704}, \bibinfo{numpages}{9}~pages.
\newblock
\showISBNx{978-1-57735-880-0}
\urldef\tempurl%
\url{https://doi.org/10.1609/aaai.v37i12.26773}
\showDOI{\tempurl}


\bibitem[Wang et~al\mbox{.}(2018a)]%
        {SYMBOLIC_INTERVAL_LINEAR_RELAXATION}
\bibfield{author}{\bibinfo{person}{Shiqi Wang}, \bibinfo{person}{Kexin Pei},
  \bibinfo{person}{Justin Whitehouse}, \bibinfo{person}{Junfeng Yang}, {and}
  \bibinfo{person}{Suman Jana}.} \bibinfo{year}{2018}\natexlab{a}.
\newblock \bibinfo{title}{Efficient Formal Safety Analysis of Neural Networks}.
\newblock , \bibinfo{numpages}{6369--6379}~pages.
\newblock
\urldef\tempurl%
\url{https://doi.org/10.48550/arXiv.1809.08098}
\showDOI{\tempurl}


\bibitem[Wang et~al\mbox{.}(2018b)]%
        {SYMBOLIC_INTERVAL}
\bibfield{author}{\bibinfo{person}{Shiqi Wang}, \bibinfo{person}{Kexin Pei},
  \bibinfo{person}{Justin Whitehouse}, \bibinfo{person}{Junfeng Yang}, {and}
  \bibinfo{person}{Suman Jana}.} \bibinfo{year}{2018}\natexlab{b}.
\newblock \showarticletitle{Formal security analysis of neural networks using
  symbolic intervals}. In \bibinfo{booktitle}{\emph{Proceedings of the 27th
  USENIX Conference on Security Symposium}} (Baltimore, MD, USA)
  \emph{(\bibinfo{series}{SEC'18})}. \bibinfo{publisher}{USENIX Association},
  \bibinfo{address}{USA}, \bibinfo{pages}{1599–1614}.
\newblock
\showISBNx{9781931971461}
\urldef\tempurl%
\url{https://doi.org/10.5555/3277203.3277323}
\showDOI{\tempurl}


\bibitem[Wang et~al\mbox{.}(2021)]%
        {BETA_CROWN}
\bibfield{author}{\bibinfo{person}{Shiqi Wang}, \bibinfo{person}{Huan Zhang},
  \bibinfo{person}{Kaidi Xu}, \bibinfo{person}{Xue Lin}, \bibinfo{person}{Suman
  Jana}, \bibinfo{person}{Cho-Jui Hsieh}, {and} \bibinfo{person}{Zico Kolter}.}
  \bibinfo{year}{2021}\natexlab{}.
\newblock \showarticletitle{Beta-CROWN: efficient bound propagation with
  per-neuron split constraints for neural network robustness verification}. In
  \bibinfo{booktitle}{\emph{Proceedings of the 35th International Conference on
  Neural Information Processing Systems}} \emph{(\bibinfo{series}{NIPS '21})}.
  \bibinfo{publisher}{Curran Associates Inc.}, \bibinfo{address}{Red Hook, NY,
  USA}, Article \bibinfo{articleno}{2289}, \bibinfo{numpages}{13}~pages.
\newblock
\showISBNx{9781713845393}
\urldef\tempurl%
\url{https://doi.org/10.48550/arXiv.2103.06624}
\showDOI{\tempurl}


\bibitem[Wang et~al\mbox{.}(2004)]%
        {SSIM}
\bibfield{author}{\bibinfo{person}{Zhou Wang}, \bibinfo{person}{A.C. Bovik},
  \bibinfo{person}{H.R. Sheikh}, {and} \bibinfo{person}{E.P. Simoncelli}.}
  \bibinfo{year}{2004}\natexlab{}.
\newblock \showarticletitle{Image quality assessment: from error visibility to
  structural similarity}.
\newblock \bibinfo{journal}{\emph{IEEE Transactions on Image Processing}}
  \bibinfo{volume}{13}, \bibinfo{number}{4} (\bibinfo{year}{2004}),
  \bibinfo{pages}{600--612}.
\newblock
\urldef\tempurl%
\url{https://doi.org/10.1109/TIP.2003.819861}
\showDOI{\tempurl}


\bibitem[Ward(1963)]%
        {HCLUSTER}
\bibfield{author}{\bibinfo{person}{Jr. Ward, Joe~H.}}
  \bibinfo{year}{1963}\natexlab{}.
\newblock \showarticletitle{Hierarchical Grouping to Optimize an Objective
  Function}.
\newblock \bibinfo{journal}{\emph{J. Amer. Statist. Assoc.}}
  \bibinfo{volume}{58}, \bibinfo{number}{301} (\bibinfo{year}{1963}),
  \bibinfo{pages}{236--244}.
\newblock
\urldef\tempurl%
\url{https://doi.org/10.1080/01621459.1963.10500845}
\showDOI{\tempurl}
\showeprint{https://www.tandfonline.com/doi/pdf/10.1080/01621459.1963.10500845}


\bibitem[Wei and Liu(2023)]%
        {ONLINE_VERIF}
\bibfield{author}{\bibinfo{person}{Tianhao Wei} {and} \bibinfo{person}{Changliu
  Liu}.} \bibinfo{year}{2023}\natexlab{}.
\newblock \bibinfo{title}{Online Verification of Deep Neural Networks under
  Domain Shift or Network Updates}.
\newblock
\newblock
\urldef\tempurl%
\url{https://doi.org/10.48550/arXiv.2106.12732}
\showDOI{\tempurl}
\showeprint[arxiv]{2106.12732}~[cs.LG]


\bibitem[Winston(1991)]%
        {BIG_M}
\bibfield{author}{\bibinfo{person}{Wayne~L. Winston}.}
  \bibinfo{year}{1991}\natexlab{}.
\newblock \showarticletitle{Operations Research: Applications and Algorithms}.
\newblock  (\bibinfo{year}{1991}).
\newblock
\urldef\tempurl%
\url{https://doi.org/10.1002/net.3230180310}
\showDOI{\tempurl}


\bibitem[Wu et~al\mbox{.}(2023)]%
        {SETS_ONLINE_LEARNING}
\bibfield{author}{\bibinfo{person}{Haoze Wu}, \bibinfo{person}{Christopher
  Hahn}, \bibinfo{person}{Florian Lonsing}, \bibinfo{person}{Makai Mann},
  \bibinfo{person}{Raghuram Ramanujan}, {and} \bibinfo{person}{Clark Barrett}.}
  \bibinfo{year}{2023}\natexlab{}.
\newblock \showarticletitle{Lightweight Online Learning for Sets of Related
  Problems in Automated Reasoning}. In \bibinfo{booktitle}{\emph{2023 Formal
  Methods in Computer-Aided Design (FMCAD)}}. IEEE, \bibinfo{pages}{1--11}.
\newblock
\urldef\tempurl%
\url{https://doi.org/10.34727/2023/ISBN.978-3-85448-060-0\_10}
\showDOI{\tempurl}


\bibitem[Wu et~al\mbox{.}(2024)]%
        {MARABOU2}
\bibfield{author}{\bibinfo{person}{Haoze Wu}, \bibinfo{person}{Omri Isac},
  \bibinfo{person}{Aleksandar Zelji{\'{c}}}, \bibinfo{person}{Teruhiro
  Tagomori}, \bibinfo{person}{Matthew Daggitt}, \bibinfo{person}{Wen Kokke},
  \bibinfo{person}{Idan Refaeli}, \bibinfo{person}{Guy Amir},
  \bibinfo{person}{Kyle Julian}, \bibinfo{person}{Shahaf Bassan},
  \bibinfo{person}{Pei Huang}, \bibinfo{person}{Ori Lahav},
  \bibinfo{person}{Min Wu}, \bibinfo{person}{Min Zhang},
  \bibinfo{person}{Ekaterina Komendantskaya}, \bibinfo{person}{Guy Katz}, {and}
  \bibinfo{person}{Clark Barrett}.} \bibinfo{year}{2024}\natexlab{}.
\newblock \showarticletitle{Marabou 2.0: A Versatile Formal Analyzer of Neural
  Networks}. In \bibinfo{booktitle}{\emph{Computer Aided Verification}},
  \bibfield{editor}{\bibinfo{person}{Arie Gurfinkel} {and}
  \bibinfo{person}{Vijay Ganesh}} (Eds.). \bibinfo{publisher}{Springer Nature
  Switzerland}, \bibinfo{address}{Cham}, \bibinfo{pages}{249--264}.
\newblock
\showISBNx{978-3-031-65630-9}
\urldef\tempurl%
\url{https://doi.org/10.1007/978-3-031-65630-9\_13}
\showDOI{\tempurl}


\bibitem[Yuan et~al\mbox{.}(2019)]%
        {ADVERSARIAL_EXAMPLES}
\bibfield{author}{\bibinfo{person}{Xiaoyong Yuan}, \bibinfo{person}{Pan He},
  \bibinfo{person}{Qile Zhu}, {and} \bibinfo{person}{Xiaolin Li}.}
  \bibinfo{year}{2019}\natexlab{}.
\newblock \showarticletitle{Adversarial Examples: Attacks and Defenses for Deep
  Learning}.
\newblock \bibinfo{journal}{\emph{IEEE Transactions on Neural Networks and
  Learning Systems}} \bibinfo{volume}{30}, \bibinfo{number}{9}
  (\bibinfo{year}{2019}), \bibinfo{pages}{2805--2824}.
\newblock
\urldef\tempurl%
\url{https://doi.org/10.1109/TNNLS.2018.2886017}
\showDOI{\tempurl}


\bibitem[Zahrawi and Shaalan(2023)]%
        {SURVEILLANCE2}
\bibfield{author}{\bibinfo{person}{Mohammad Zahrawi} {and}
  \bibinfo{person}{Khaled Shaalan}.} \bibinfo{year}{2023}\natexlab{}.
\newblock \showarticletitle{Improving video surveillance systems in banks using
  deep learning techniques}.
\newblock \bibinfo{journal}{\emph{Scientific Reports}}  \bibinfo{volume}{13}
  (\bibinfo{date}{05} \bibinfo{year}{2023}), \bibinfo{pages}{16}.
\newblock
\urldef\tempurl%
\url{https://doi.org/10.1038/s41598-023-35190-9}
\showDOI{\tempurl}


\bibitem[Zhang et~al\mbox{.}(2018)]%
        {LPIPS}
\bibfield{author}{\bibinfo{person}{Richard Zhang}, \bibinfo{person}{Phillip
  Isola}, \bibinfo{person}{Alexei~A. Efros}, \bibinfo{person}{Eli Shechtman},
  {and} \bibinfo{person}{Oliver Wang}.} \bibinfo{year}{2018}\natexlab{}.
\newblock \showarticletitle{The Unreasonable Effectiveness of Deep Features as
  a Perceptual Metric}. In \bibinfo{booktitle}{\emph{2018 {IEEE} Conference on
  Computer Vision and Pattern Recognition, {CVPR}}}.
  \bibinfo{publisher}{Computer Vision Foundation / {IEEE} Computer Society},
  \bibinfo{pages}{586--595}.
\newblock
\urldef\tempurl%
\url{https://doi.org/10.48550/arXiv.1801.03924}
\showDOI{\tempurl}


\bibitem[Zheng et~al\mbox{.}(2025)]%
        {ACCMILP}
\bibfield{author}{\bibinfo{person}{Fei Zheng}, \bibinfo{person}{Qingguo Xu},
  \bibinfo{person}{Zhou Lei}, {and} \bibinfo{person}{Huaikou Miao}.}
  \bibinfo{year}{2025}\natexlab{}.
\newblock \showarticletitle{AccMILP: An Approach for Accelerating Neural
  Network Verification Based on Neuron Importance}. In
  \bibinfo{booktitle}{\emph{Engineering of Complex Computer Systems}},
  \bibfield{editor}{\bibinfo{person}{Guangdong Bai}, \bibinfo{person}{Fuyuki
  Ishikawa}, \bibinfo{person}{Yamine Ait-Ameur}, {and}
  \bibinfo{person}{George~A. Papadopoulos}} (Eds.).
  \bibinfo{publisher}{Springer Nature Switzerland}, \bibinfo{address}{Cham},
  \bibinfo{pages}{88--107}.
\newblock
\showISBNx{978-3-031-66456-4}
\urldef\tempurl%
\url{https://doi.org/10.1007/978-3-031-66456-4_5}
\showDOI{\tempurl}


\bibitem[Zhou et~al\mbox{.}(2024)]%
        {BICCOS}
\bibfield{author}{\bibinfo{person}{Duo Zhou}, \bibinfo{person}{Christopher
  Brix}, \bibinfo{person}{Grani~A Hanasusanto}, {and} \bibinfo{person}{Huan
  Zhang}.} \bibinfo{year}{2024}\natexlab{}.
\newblock \bibinfo{title}{Scalable Neural Network Verification with
  Branch-and-bound Inferred Cutting Planes}.
\newblock
\newblock
\urldef\tempurl%
\url{https://doi.org/10.48550/ARXIV.2501.00200}
\showDOI{\tempurl}
\showeprint[arxiv]{2501.00200}~[cs.LG]


\bibitem[Zhu and Tan(2020)]%
        {THOMPSON_MVTS}
\bibfield{author}{\bibinfo{person}{Qiuyu Zhu} {and} \bibinfo{person}{Vincent
  Y.~F. Tan}.} \bibinfo{year}{2020}\natexlab{}.
\newblock \showarticletitle{Thompson sampling algorithms for mean-variance
  bandits}. In \bibinfo{booktitle}{\emph{Proceedings of the 37th International
  Conference on Machine Learning}} \emph{(\bibinfo{series}{ICML'20})}.
  \bibinfo{publisher}{JMLR.org}, Article \bibinfo{articleno}{1075},
  \bibinfo{numpages}{10}~pages.
\newblock
\urldef\tempurl%
\url{https://doi.org/10.48550/arXiv.2002.00232}
\showDOI{\tempurl}


\end{thebibliography}
\ifthenelse{\EXTENDEDVER<0}{
\newpage
\appendix
\section{Running Example}\label{sec:runex}
We next describe a running example of \Cref{algo::batch_verif}, given an MNIST 3$\times$100 fully connected classifier with 3 hidden layers, ten MNIST images $S=\{x_1,\ldots,x_{10}\}$, class $c = 0$ and $\epsilon=0.1$.
\tool begins by identifying that all inputs are classified as $c$.
Then, it determines whether to split in layer $\ell=1$ or $\ell=2$ as follows. 
First, it samples an input $x_5$ and verifies its $\epsilon$-ball when splitting the network at layer $1$.
Then, it samples $x_3$ and verifies its $\epsilon$-ball when splitting the network at layer $2$.
Since the verification runtime of $x_3$ is shorter, it sets $\ell=2$. Both $\epsilon$-balls are robust, \tool updates their status in \texttt{is\_robust} and removes these inputs from $S$.
\tool continues by computing the activation patterns and constructing the binary tree using H-Cluster. Then, it initializes the $\mathcal{MAB}$ agent and begins iterations while the tree is not empty. In the first iteration, the $\mathcal{MAB}$ agent returns $k = 6$. \tool performs a pre-order search and returns the mini-batch $B = \{x_2, x_4, x_6, x_7, x_{10} \}$. Then, it computes the bounds up to the layer $\ell=2$ for every input in $B$.
Then, it performs a batch analysis from layer 3 till the output layer. The MILP solver returns a counterexample, where $I_{x_2}=1$. Thus, \tool checks whether $B_\epsilon^\infty(x_2)$ is robust and expedites its analysis by leveraging the already computed bounds from the input layer till layer $\ell$.
The MILP solver finds a counterexample, indicating that $x_2$ is \emph{Non-Robust}. Then, \tool removes $x_2$ from $B$, adds the constraint $I_{x_2}=0$ and calls the MILP solver again. The solver returns there is no counterexample, thus \tool determines that all inputs $\{x_4, x_6, x_7, x_{10} \}$ are robust. It then updates their status in \texttt{is\_robust} and updates the distribution of batch size 5, based on the velocity of $B$. The velocity is $\frac{4}{25}$, where the denominator is the total time of the
analysis, except for the additional analysis time of $x_2$.
Then, \tool begins another iteration. The $\mathcal{MAB}$ agent returns $k=4$. Thus, \tool adds all remaining inputs to the mini-batch $B = \{x_1, x_8, x_9\}$. \tool continues similarly to the previous iteration and identifies that all $\epsilon$-balls are robust, thus it updates their status. It updates the distribution of batch size 3 based on the velocity $\frac{3}{23}$. 
At the end of this iteration, $\mathcal{T}$ is empty, all inputs in $S$ are determined as robust or not and \tool terminates.

\section{Proof}\label{sec:proof}
\ftc*
\begin{proof}
    In the first direction, we assume that the MILP is feasible.
    Since $\forall i \in [k]: I_i \in \{0, 1\} $, \Cref{logical_disjunction_1} implies that there exists $i \in [k]$ such that $I_i = 1$, and $\forall j \neq i \in [k]: I_j = 0 $.
    From assigning the binary values in~\Cref{logical_disjunction_2}, it holds that $y \geq l_i$ and 
    for every $j\neq i.\ y \geq 0$.
    Additionally, from assigning them in~\Cref{logical_disjunction_3}, it holds that $y \leq u_i$ and for every $j\neq i.\ y \leq u_M$.  Because $l_i \geq 0$ and $u_i \leq u_M$, we obtain $y \in [l_i, u_i]$.
    
    In the second direction, we assume that there exists $y' \in [l_i, u_i]$ for some $i\in [k]$.
    Consider the assignment of $I_i = 1$, $\forall j \neq i \in [k]: I_j = 0 $, and $y = y'$.
    We show that it satisfies the above MILP.
    Clearly, \Cref{logical_disjunction_1} is satisfied.
    By the assumption, it holds that $l_i \leq y' \leq u_i$.
    Hence, $y \geq l_i \cdot 1 = l_i \cdot I_i$ and $y \leq u_i \cdot 1 + u_M \cdot 0 = u_i \cdot I_i + u_M \cdot (1 - I_i)$.
    For every $j \neq i$, we set $I_j = 0$.
    Since $y' \geq l_i \geq 0$, it follows that $y \geq l_j \cdot 0 = l_j \cdot I_j$.
    Additionally, $y' \leq u_i \leq u_M$, thus $y \leq u_j \cdot 0 + u_M \cdot 1 = u_j \cdot I_j + u_M \cdot (1 - I_j)$.
    Therefore, \Cref{logical_disjunction_2} and \Cref{logical_disjunction_3} are true for every $j \in [k]$.
\end{proof} 
}{}
\end{document}